%% file: iclr2027_conference.tex
\documentclass{article} 
\usepackage{iclr2027_conference,times}

\input{math_commands.tex}

\usepackage{hyperref}
\usepackage{url}
\usepackage{cleveref}
\usepackage{mathtools}
\usepackage[toc,page]{appendix}
\usepackage{titletoc}

\newtheorem{problem}{Problem}
\newtheorem{definition}{Definition}

\newtheorem{lemma}{Lemma}
\newtheorem{assumption}{Assumption}

\newtheorem{proof}{Proof}

\newtheorem{example}{Example}
\newtheorem{theorem}{Theorem}

\usepackage{algorithm}
\usepackage{algorithmic}
\usepackage{array}
\usepackage{booktabs}
\usepackage{xcolor}

\usepackage{tikz}
\usetikzlibrary{positioning,arrows.meta}
\usepackage{amssymb}

\title{Targeted Active Learning for Preference-\\Based Treatment Effects on Multivariate\\ Outcomes}

\author{%
Lola Giordani$^{1,2}$ \quad Mathieu Even$^{1,3}$ \quad Chlo\'e Geoffroy$^{1}$ \quad Jean-Christophe Corvol$^{2}$ \\
\bfseries Rapha\"el Porcher$^{4}$ \quad Federico Pavone$^{1}$ \\[0.6em]
\normalfont
$^{1}$Theremia, Paris \\
$^{2}$Paris Brain Institute (ICM), Sorbonne Universit\'e, Inserm, CNRS, AP-HP \\
$^{3}$Inria, Inserm, Universit\'e de Montpellier, France \\
$^{4}$Universit\'e Paris Cit\'e, Inserm, Centre for Research in Epidemiology and Statistics (CRESS) \\[0.3em]
}

\iclrfinalcopy 
\begin{document}

\maketitle
\lhead{} 

\begin{abstract}
Treatment efficacy is traditionally demonstrated on the basis of a single primary outcome. However, clinical decision-making usually requires consideration of multiple outcomes, balancing expected benefits against potential risks. The relative value assigned to these outcomes varies substantially from one patient to another. Given a \emph{preference rule} over outcome profiles, treatment effects and optimal policies can be defined and estimated. Such a rule is rarely available in practice: it must itself be estimated from pairwise comparisons of outcome profiles, which are costly to collect from clinical experts. We propose an active learning framework that selects which comparisons to query. Standard criteria maximize the information gained on the preference rule itself. We instead target the quantities of interest, and select the query that most reduces uncertainty on the treatment effect and on the optimal policy induced by the learned rule. Under a Gaussian process model of the preference rule, we derive a closed-form approximation of this criterion. On semi-synthetic data built from a Parkinson's disease cohort with 13 clinical outcomes, our criterion achieves lower treatment effect estimation error and lower policy regret than existing criteria at equal query budget.
\end{abstract}

\section{Introduction}
 
A large literature \citep{rubin1974estimating, angrist1995identification,nie2021quasi} answers  whether a treatment increases or decreases one quantity of interest, such as the outcome of a disease.
However many clinical decisions involve several outcomes at once: a treatment may
improve one symptom while worsening another, or bring a clinical benefit at the
cost of a side effect. 
While one can estimate an effect for each outcome on its own, the problem becomes more difficult when several separate and competing effects have to be turned into a single treatment decision.
Standard practice avoids this by reducing the outcome to a scalar before the causal question is even posed, either by picking a single primary endpoint or by combining outcomes into a composite score \citep{freemantle_Composite_2003}. Such a score settles the trade-offs in advance for the whole population, while the relative
importance of each outcome varies across patients and clinical contexts.
Parkinson's disease illustrates this well, by combining motor symptoms, such as tremor, rigidity and akinesia, with non-motor symptoms such as sleep disorders, depression, cognitive decline and dysautonomia \citep{kalia_Parkinsons_2015}. It is usually treated with dopamine replacement therapy, which improves motor symptoms but comes with its own costs: motor complications such as fluctuations and dyskinesia, non-motor complications such as impulse control disorders and hallucinations, and a worsening of existing non-motor symptoms such as dysautonomia. Treatments act differently on each of these dimensions, and no standard scale says how to weigh them against each other.
Clinicians naturally balance them when selecting a therapy but classical treatment effect estimation methods don't. Parkinson's disease is used as a running example throughout the paper.

Instead of measuring how much a treatment moves some fixed score, \citet{even_Rethinking_2026} ask which complete outcome profile a clinician would judge to be better overall. They assume access to a known
preference rule, which quantifies for any two outcome profiles how likely it is that the first would be judged preferable to the second.
The preference rule captures both that a single clinician is not perfectly consistent and that different clinicians may disagree. \citet{parnas_Preferencebased_2026b} then define \emph{conditional preference treatment effect} (CPTE) as the probability that, among patients sharing the same characteristics, the outcome experienced under treatment would be preferred to
the outcome experienced under control. This yields a treatment effect for outcomes with several dimensions, or with categories that can be ranked but not subtracted.

\subsection{Contributions and outline of the paper}

The preference rule is known in some cases, for instance for hierarchical outcomes \citep{pocock_win_2012a}, but in general the trade-offs between symptoms, side effects and quality of life have to be elicited. Clinicians or other experts are shown pairs of outcome profiles and asked which one they prefer, and the rule is learned from these labeled pairs.
Each comparison takes clinician time, so only a small and somewhat noisy set of judgments can be collected.
Two difficulties follow. First, the uncertainty in the learned preference rule propagates to the CPTE, which becomes a random quantity rather than a fixed number.
Second, the choice of which comparisons to ask becomes crucial. Existing selection methods target either the preference model itself or the answers to future comparisons, but not the causal estimand, which integrates preferences over the distributions of potential outcomes.
In this paper, we address both difficulties with, to our knowledge, the first active learning method for preference elicitation that targets a causal estimand. We treat the CPTE as a random quantity under the posterior of the preference model, and we select the comparisons that are expected to be most informative about it. Clinician time is then spent on the trade-offs that decide which treatment is better for the target patients, rather than on learning the preference rule everywhere.
After presenting the background (\Cref{sec:background}) and formalizing the active learning problem (\Cref{sec:active_learning_problem}), we make the following contributions:

\begin{itemize}
\itemsep2pt \parskip0pt
\item We formulate the CPTE induced by a probabilistic preference model as a
posterior random variable $\Psi_x$, which captures the uncertainty that limited preference data leaves in the estimand (Section~\ref{sec:model}). We then linearise this estimand around the posterior mean utility to
obtain a tractable, Gaussian characterisation of its law, with an error
bound computable at every round from the current posterior
(Section~\ref{sec:linearization}).
\item We introduce the \textit{CPTE -- Expected Information Gain} (\textsc{CPTE-EIG}), an acquisition criterion that scores
candidate comparisons by their expected information gain about the causal
estimand in a target population. Under the linearised CPTE, we derive a computationally efficient expression for it,
in closed form up to a one-dimensional quadrature (Section~\ref{sec:closed-form}).
\item We evaluate the framework in controlled experiments and on a
semi-synthetic Parkinson's disease cohort with thirteen outcomes, and show
that at equal query budget it reduces both CPTE estimation error and the
regret of the resulting treatment policy (Section~\ref{sec:experiments}).
\end{itemize}

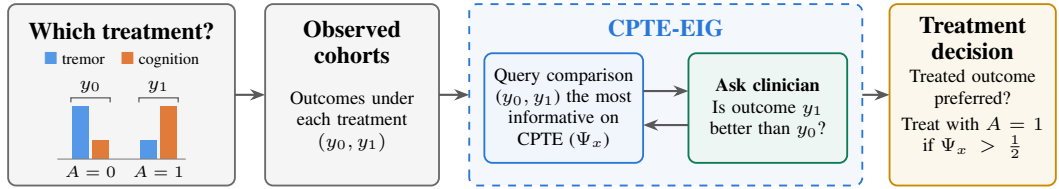
\begin{figure}[h!]
\centering
\begin{tikzpicture}[
  x=1mm, y=1mm, font=\scriptsize,
  box/.style={draw=#1, fill=#1!7, rounded corners=3pt, line width=0.7pt},
  txt/.style={align=center, text width=#1},
  arr/.style={-{Stealth[length=2mm]}, line width=0.8pt, draw=black!65},
  ttl/.style={font=\footnotesize\bfseries, align=center},
]
\definecolor{wfgray}{HTML}{6b6b6b}
\definecolor{wfblue}{HTML}{2a78d6}
\definecolor{wfgreen}{HTML}{0f7a55}
\definecolor{wfgold}{HTML}{c98a00}
\definecolor{wftremor}{HTML}{5598e7}
\definecolor{wfcog}{HTML}{e07b39}

\draw[box=wfgray] (0,0) rectangle (30,24);
\node[ttl] at (15,20.5) {Which treatment?};
\node[font=\tiny] at (15,16.6) {\textcolor{wftremor}{$\blacksquare$} tremor\quad
\textcolor{wfcog}{$\blacksquare$} cognition};
\foreach \x/\a/\b in {8.5/7/2.5, 17.5/2.5/7} {
\fill[wftremor] (\x,3.5) rectangle ++(2.2,\a);
\fill[wfcog] (\x+2.6,3.5) rectangle ++(2.2,\b);
}
\draw[black!50] (6.5,3.5) -- (23.5,3.5);
\foreach \x/\y in {8.5/y_0, 17.5/y_1} {
\draw[black!60, line width=0.5pt] (\x-0.2,11) -- ++(0,0.6) -- ++(5.2,0) -- ++(0,-0.6);
\node[font=\scriptsize] at (\x+2.4,13.1) {$\y$};
}
\node[font=\tiny] at (10.9,1.8) {$A=0$};
\node[font=\tiny] at (19.9,1.8) {$A=1$};

\draw[box=wfgray] (34,0) rectangle (57,24);
\node[ttl, txt=22mm] at (45.5,19) {Observed\\ cohorts};
\node[txt=22mm] at (45.5,8.5) {Outcomes under\\ each treatment\\ $(y_0,y_1)$};

\draw[box=wfblue, dashed, fill=wfblue!5] (61,0) rectangle (113,24);
\node[ttl, text=wfblue] at (87,20.8) {CPTE-EIG};
\draw[box=wfblue] (63,3) rectangle (84,17.5);
\node[txt=20.5mm] at (73.5,10.25) {Query comparison\\ ($y_0,y_1$) the most\\ informative on CPTE ($\Psi_x$)};
\draw[box=wfgreen] (90,3) rectangle (111,17.5);
\node[txt=20.5mm] at (100.5,10.25) {\textbf{Ask clinician}\\ Is outcome $y_1$\\ better than $y_0$?};
\draw[arr] (84,12.5) -- (90,12.5);
\draw[arr] (90,8) -- (84,8);

\draw[box=wfgold] (116.6,0) rectangle (138.6,24);
\node[ttl] at (127.6,19.5) {Treatment \\ decision};
\node[txt=21.5mm] at (127.6,9.5) {Treated outcome\\ preferred?\\[2pt] Treat with $A=1$ if $\Psi_x>\frac12$};

\draw[arr] (30,12) -- (34,12);
\draw[arr] (57,12) -- (61,12);
\draw[arr] (113,12) -- (116.6,12);
\end{tikzpicture}
\caption{Overall workflow presented in the paper: CPTE-EIG asks the clinician only the comparisons that matter for the treatment decision.}
\label{fig:workflow}
\end{figure}

\subsection{Related works}

\paragraph{Comparing treatments on several outcomes.}
Weighing benefits and risks across several outcomes at once is a recurrent problem in clinical research and drug regulation. Multi-criteria decision analysis offers a structured way to combine several outcomes into a single recommendation \citep{gongora-salazar_Use_2023}. Composite endpoints do this by combining outcomes into one weighted sum (the composite outcome), so everything sits on the same scale \citep{freemantle_Composite_2003}. Composite outcomes can also be learned on data in a covariate-dependent way \citep{butler2018incorporating}. Generalized pairwise comparisons
(GPC) \citep{buyse_Generalized_2010} and the win ratio \citep{pocock_win_2012a} avoid this by ranking outcomes by priority and comparing them hierarchically, one pair of patients at a time, moving to a lower-priority outcome only when higher-priority ones are tied. Either way, a clinician or any specialist, has to fix the weights or the ranking in advance, and the result is usually a single population-level number rather than a causal effect that varies with a patient's characteristics. \citet{even_Rethinking_2026} show this matters for the win ratio: the historical estimand can reverse treatment recommendations under heterogeneous populations, and they propose an individual-level fix, though the ranking itself stays fixed. \citet{petit2025optimal} and \citet{parnas_Preferencebased_2026b} push this individual-level view from estimation to decision-making, defining an optimal individualized treatment rule within the same prioritized-outcome GPC framework.

\paragraph{Learning preferences from comparisons.}
A standard way to model pairwise preferences is to assume each option has
a hidden value, and that people prefer whichever option is worth more
\citep{bradley_Rank_1952}. \citet{chu2005preference} turned this into a
Bayesian model: a Gaussian process describes the hidden value function, and
each comparison updates a posterior over it, giving not just which item is
preferred but how confident that ordering is. Later work extends this in
several directions: handling clinicians who may disagree with each other
\citep{houlsby_Collaborative_2012}, noisy or self-contradictory judgments
\citep{chau_Learning_2022}, and answers that depend on the context in which
a comparison is asked \citep{bengs_Stochastic_2022}. In every case, the
object being learned is the preference function itself.

\paragraph{Choosing which comparisons to ask.}
Since each comparison takes clinician time, the comparisons should be
chosen carefully rather than at random. Much of the work on active
preference learning aims at the best option.
\citet{eric_Active_2007} ask as few questions as possible to find the item
a person values most, and
preferential Bayesian optimisation \citep{gonzalez_Preferential_2017}
searches for the maximiser of a latent function from pairwise comparisons.
In these works, the goal
is the best option or the preference function itself.
Our method belongs to Bayesian experimental design
\citep{lindley1956measure,chaloner1995bayesian,rainforth2024modern}, which
selects the query with the largest expected information gain about a
quantity of interest. Existing criteria differ in the quantity they target.
Bayesian active learning by disagreement (BALD)
\citep{houlsby_Bayesian_2011b}, a widely used criterion for preference
learning, targets the model parameters.
Expected predictive information gain (EPIG)
\citep{smith_PredictionOriented_2023b} instead targets the predictions on
inputs drawn from a chosen target distribution.
For treatment effects, Causal-BALD \citep{jesson_CausalBALD_2021} and
Causal-EPIG \citep{gao_CausalEPIG_2025a} choose the individuals whose
outcomes are measured so as to reduce uncertainty respectively in the conditional average
treatment effect estimator. Our setting
differs from both in two ways. First, we collect noisy
pairwise judgments rather than outcome measurements, and use them to learn
a preference function rather than an outcome model. Second, we target
neither the preference function nor a treatment effect on a scalar outcome,
but a treatment effect that is defined only once the preference function is
combined with the distribution of the multivariate potential outcomes.

\section{Background}
\label{sec:background}

Throughout the paper we use two different sources of data answering different questions: \textit{(i)} a cohort of patients tells us how treatments change outcomes, while \textit{(ii)} a set of pairwise comparisons, in which clinicians state which of two outcome profiles they prefer, tells us how outcomes should be compared. The first is fixed and available in advance, and is subject to the assumptions we describe next in \Cref{sec:data}. The second is collected sequentially by our method, is indexed independently of the cohort, and is introduced in Section~\ref{sec:active_learning_problem}.

\subsection{Notations and assumptions}
\label{sec:data}

\begin{table}[t]
\begin{center}
\begin{small}
\caption{The quantities of the paper instantiated for Parkinson's disease (PD),
chosen in accordance with specialists. MDS-UPDRS
\citep{goetz_Movement_2008} is the standard clinical scale for rating PD.}
\label{tab:parkinson}
\setlength{\tabcolsep}{4pt}
\begin{tabular}{@{}>{\raggedright\arraybackslash}p{0.16\columnwidth}>{\raggedright\arraybackslash}p{0.8\columnwidth}@{}}
\toprule
Quantity & Parkinson's disease instantiation \\
\midrule
$X$
& Full covariate. Includes comorbidities, PD subtype/staging, cognitive
and psychiatric scales, anxiety, fatigue,
biomarkers/genetics, ethnicity, race, education, family history, BMI. \\
\midrule
$\tilde{X} = \gamma(X)$
& sex, age group ($<50$ to $80{+}$, $5$-year bins), disease duration \\
\midrule
$Y \in \mathcal{Y} \subseteq \mathbb{R}^{13}$
& $13$ MDS-UPDRS items. \emph{Non-motor} ($7$): cognition, hallucinations,
mood, apathy, dopamine dysregulation, sleep, orthostatic symptoms.
\emph{Motor} ($4$): self-care, tremor, walking/balance, freezing.
\emph{Motor complications} ($2$): dyskinesia and fluctuation impact \\
\midrule
$A$
& $0$ (levodopa) or $1$ (dopamine agonists) \\ 
\bottomrule
\end{tabular}
\end{small}
\end{center}

\end{table}

Each participant in the cohort is described by covariates $X \in \mathcal{X}$, a binary treatment $A \in \{0,1\}$ and an outcome $Y \in \mathcal{Y}$, with potential outcomes $Y(0), Y(1)$. The outcome is multivariate, $\mathcal{Y} \subseteq \mathbb{R}^p$, each coordinate being a bounded ordinal score for one clinical dimension. These coordinates are not required to be on a common scale. Table~\ref{tab:parkinson} specifies the quantities we introduce for Parkinson's disease as a concrete example.
We work under the following standard causal assumptions \citep{rosenbaum_Central_1983,rubin_Randomization_1980}.

\begin{assumption}
[Causal Assumptions]
\label{ass:identification}
\textit{(Consistency \& No interference)} $Y = Y(A)$;
\textit{(Unconfoundedness)} $\bigl(Y(0), Y(1)\bigr) \perp A \mid X$;
\textit{(Positivity)} $0 < \mathbb{P}(A = 1 \mid X = x) < 1$ for
 all $x$.
\end{assumption}

Under \Cref{ass:identification}, the conditional laws of $Y(0)$ and $Y(1)$ given $X=x$
are identifiable from observed data \citep{hernan_Causal_2025}. We write
$\hat{\mu}_a(\cdot \mid x)$, $a \in \{0,1\}$ for their estimates.
Let
\begin{equation}
\label{eq:nu}
\nu_x(\mathrm{d}z) = \hat{\mu}_1(\mathrm{d}y_1 \mid x) \otimes \hat{\mu}_0(\mathrm{d}y_0 \mid x)\,,\qquad z=(y_1,y_0)\,,
\end{equation}
be the joint law of an independent treated--control pair sharing same covariates $x$. Each
$\hat{\mu}_a$ is a joint distribution over the full outcome vector:
replacing it with a product of coordinate marginals would discard
within-arm dependence between clinical dimensions and change the law
of $z$.

\subsection{Comparing outcome profiles}
\label{sec:preferences}

Informing about how two outcome profiles $y_0$ and $y_1$ compare is a clinical judgment, elicited directly from clinicians, who are shown the two profiles without the treatment that produced them.
A clinical pairwise comparison may require context knowledge about the patient for whom the comparison is made, such as a set of features.
Showing a clinician the full covariate $x$ of a patient together with the compared outcomes is neither feasible nor right in principle, since some components of $x$ affect how well a treatment works without affecting how outcomes should be traded off against each other. We therefore elicit preferences using only a fixed, clinically justified summary of the patient, keeping $x$ for identification alone.
Let
$\gamma: \mathcal{X} \to \tilde{\mathcal{X}}$ be a fixed, known map, and write $\tilde x = \gamma(x) \in \tilde{\mathcal X}$
for the resulting \emph{query covariate}.
Table~\ref{tab:parkinson} gives $\gamma$ for Parkinson's disease as a concrete example: tolerance for a side effect legitimately changes with age, whereas ethnicity may help predict how well a treatment works, and so remains in $x$, but has no place in deciding whether tremor is preferable to cognitive impairment.

Following \citet{parnas_Preferencebased_2026b}, we summarise how a
clinician ranks two outcome profiles by a \emph{preference rule} $w$,
and extend it to query covariates.
In \citet{parnas_Preferencebased_2026b}, $w$ accounts
both for a single clinician not being deterministic and for disagreement
between clinicians. In this work, we model a single rule describing an average
clinician, and do not model individual raters separately.

\begin{definition}[Preference rule]
    \label{def:preference_rule}
    The \emph{preference rule} $w$ is defined as the probability taken over clinicians' judgments that $y_1$ is preferred to $y_0$ when both are shown
at query covariate $\tilde x$: for $y_1, y_0 \in \mathcal{Y}$ and
$\tilde{x} \in \tilde{\mathcal{X}}$,
$w(y_1 \mid y_0, \tilde{x}) = \mathbb{P}(y_1 \succ y_0 \mid \tilde{x})$. We assume the no ties assumption: $w(y_1 \mid y_0, \tilde{x}) + w(y_0 \mid y_1, \tilde{x}) = 1$, for every $y_0,y_1,\tilde x$.
\end{definition}

\subsection{The conditional preference treatment effect}
\label{sec:cpte}

With $w$ fixed, we can now state the causal target of the paper.

\begin{definition}[CPTE and $\Psi_x$]
\label{def:cpte}
For a fixed rule $w$, the conditional preference treatment effect (CPTE)
\citep{parnas_Preferencebased_2026b} at covariate $x \in \mathcal{X}$ is:
\begin{equation}
\label{eq:cpte}
\mathrm{CPTE}_w(x) = \mathbb{E}\bigl[\, w\bigl(Y_i(1) \mid Y_j(0),
\gamma(x)\bigr) \;\big|\; X_i = X_j = x \,\bigr],
\end{equation}
with $i \neq j$.
The expectation is over two independent copies of the
potential outcomes at $x$.

Under Assumption~\ref{ass:identification}, the CPTE is identified from observed data as $\mathbb{E}[w(Y_i \mid Y_j, \gamma(x)) \mid X_i = X_j = x,\, A_i = 1,\, A_j = 0]$ \citep{parnas_Preferencebased_2026b}, or equivalently as the integral of $w(\cdot \mid \cdot, \gamma(x))$ against $\nu_x$, previously defined in \Cref{eq:nu}:
\begin{equation}
\label{eq:psi}
\Psi_x \;=\; \mathrm{CPTE}_w(x) \;=\; \int_{\mathcal{Y}^2} w(y_1 \mid y_0,
\gamma(x)) \, \nu_x(dz).
\end{equation}
\end{definition}

$\Psi_x$ is the probability that a treated outcome is preferred to a
control outcome among patients with covariate $x$, and $\Psi_x > 1/2$
means that treatment is preferred at $x$. Taking $\mathcal{Y} \subseteq
\mathbb{R}$ and $w(y_1 \mid y_0, \tilde{x}) = y_1 - y_0$ recovers the
CATE, so the CPTE generalises it to
outcomes that can be compared but not substracted.
Under the no-ties assumption, the expected preference for control at $x$ is
$1 - \Psi_x$. The policy maximising expected preference is therefore $\pi^\star(x) = \mathbf{1}\{\Psi_x > 1/2\}$, and any policy $\pi$ incurs regret $\mathbb{E}_X\bigl[\, \lvert 2\Psi_X - 1 \rvert \cdot \mathbf{1}\{\pi(X) \neq \pi^\star(X)\} \,\bigr]$: decision errors are weighted by how clear-cut the decision is at $X$.

\section{Active learning problem}
\label{sec:active_learning_problem}

In our use case, the rule $w$ is not given: it has to be elicited from clinicians, one pairwise comparison at a time, asked at the query
covariate $\tilde{x}$ rather than at the full covariate $x$.

\begin{definition}[Query]
\label{def:query}
A query is a triplet $c = (\tilde{x}, a, b)$, in which a clinician
compares outcome profiles $a, b \in \mathcal{Y}$ presented at query
covariate $\tilde{x}$, and returns a binary answer $V_c \in \{0, 1\}$,
with $V_c = 1$ if $a$ is preferred. 
\end{definition}

Comparisons are collected one at a time: at round $t$ the method selects a
query $c_{t+1}$ from a candidate pool $\mathcal{C}_t$, observes a single
answer $v_{t+1}$, and updates its beliefs about $w$. The pool is finite
and built from outcome profiles observed in the cohort, so that the
profiles shown to a clinician remain clinically realistic. The
procedure stops after a budget of $T$ queries, reflecting available
clinician time, and we write $\mathcal{D}_t = \{(c_k, v_k)\}_{k=1}^t$ for
the comparisons observed after $t$ rounds.
We write $p^\star$ for a distribution over full covariates $x^\star$, describing the population on which treatment decisions will be made. In practice, $p^\star$ is the empirical distribution of the target
cohort. 
This yields the following problem, which our paper addresses.

\begin{problem}
\label{prob:main}
Given a budget of $T$ comparisons and the estimated potential-outcome laws
$\hat{\mu}_a$ of Section~\ref{sec:data}, choose queries $c_1, \dots, c_T
\in \mathcal{C}_t$ so that the resulting CPTE estimate and the policy it
induces are as accurate as possible on $p^\star$.
\end{problem}

Active learning tackles the following question: which query should we
ask next? 
As in Bayesian experimental design, we choose the query with the largest
expected information gain about a quantity of interest. Here, this
quantity is $\Psi_x$, the proportion of patients with covariate $x$ for
whom treatment is preferred, which alone determines the optimal policy
$\pi^\star$ and its regret.  

\begin{definition}[CPTE-EIG]
\label{def:cpte-eig}
For a target covariate $x$ and a candidate query $c$, the expected
information gain (EIG) of $c$ about $\Psi_x$ is the following mutual information
\citep{cover_Elements_1991}: $\mathrm{EIG}_{\Psi(x)}(c) = \mathbb{I}(\Psi_x; V_c \mid \mathcal{D}_t, c)$.
The CPTE-EIG score of $c$ averages this over the target population,
and the next query maximises it:
\begin{equation}
\label{eq:cpte-eig}
\mathrm{CPTE\text{-}EIG}(c) = \mathbb{E}_{x^\star \sim p^\star}\bigl[\, \mathrm{EIG}_{\Psi(x^\star)}(c) \,\bigr], \quad c_{t+1} \in \arg\max_{c \in \mathcal{C}_t} \mathrm{CPTE\text{-}EIG}(c).
\end{equation}
\end{definition}

Equation~\ref{eq:cpte-eig} is well defined for any probabilistic model
of $w$, but in general does not admit a closed form or cannot be computed efficiently: evaluating it would require sampling $w$ from its posterior and integrating
against $\nu_{x^\star}$ for every candidate query and target
covariate at every round, by Monte Carlo
\citep{robert_Monte_2011}. We next resolve this difficulty by proposing a computationally efficient approach.

\section{Efficient CPTE-EIG computation under Gaussian processes}
\label{sec:model}

\textbf{Utility contrast models.}
While the active learning criterion CPTE-EIG is well-defined for any
probabilistic model of $w$, we specify a model under which we can derive closed form updates. 
Following the standard Bradley-Terry model \citep{bradley_Rank_1952}, we assume preferences
come from a latent utility that assigns a single score per outcome profile. One profile is preferred to another whenever its utility is higher. We thus define
\label{def:utility}
$u: \tilde{\mathcal{X}} \times \mathcal{Y} \to \mathbb{R}$ and for $r=(r_1,r_0)\in (\tilde{\mathcal{X}} \times \mathcal{Y})^2$, we define the \emph{utility contrast} as:
\begin{equation}
\label{eq:contrast}
D_r u = u(r_1) - u(r_0).
\end{equation}
A query $c=(\tilde x,a,b)$ corresponds to $r=((\tilde x,a),(\tilde
x,b))$, which gives the notation $D_c u$. A treated--control pair $z=(y_1,y_0)$ at full
covariate $x$ (as in \Cref{eq:nu}) corresponds to
$r=((\tilde x,y_1),(\tilde x,y_0))$, giving
$
\label{eq:contrast-xz}
D_{x,z} u = u(\tilde x, y_1) - u(\tilde x, y_0),
$
where the index carries the full covariate $x$, but $u$ itself is
evaluated at $\tilde x=\gamma(x)$.
The clinician's noisy answer and the CPTE both compare an outcome pair
through the same probit link \citep{thurstone_law_1927}, and we tie
them to a single noise scale $\lambda$. Formally,
\begin{align}
V_c \mid u, c \sim \mathrm{Bernoulli}\left(\left(\frac{D_c u}{\lambda}\right)\right), \qquad
w_u(x,z) = \Phi\left(\frac{D_{x,z} u}{\lambda}\right), \label{eq:likelihood_w-u}
\end{align}
where $\Phi$ is the standard normal cumulative distribution function. A large $\lambda$ drives both towards a coin flip. A small $\lambda$ makes each nearly deterministic in the sign of its contrast. Because $\Phi$ is antisymmetric,
$\Phi(t)+\Phi(-t)=1$, swapping $y_1$ and $y_0$ in \Cref{eq:likelihood_w-u} flips
the sign of the contrast and gives $w_u(\tilde x,(y_0,y_1))=1-w_u(\tilde
x,(y_1,y_0))$: the no-ties condition holds for every fixed $u$.

\textbf{Utililty prior.}
We treat $u$ as random and place a Gaussian process prior on it
\citep{rasmussen_Gaussian_2005a} with marginal variance $\sigma^2$, and
write $U$ for the resulting random function, reserving $u$ for a fixed
realisation. Since \Cref{eq:likelihood_w-u} is not Gaussian, the exact
posterior is unavailable in closed form. As is standard for Gaussian
process preference models \citep{chu2005preference}, we use a Laplace
approximation \citep{williams_Bayesian_1998} which gives
$U\mid\mathcal D_t\approx q_t=\mathrm{GP}(m_t,k_t)$, and take all
probabilities, expectations and covariances involving $U$ under $q_t$
from here on. The contrast operator is linear, so every contrast
$D_rU$ is Gaussian under $q_t$. We write $m_t^\Delta(r)$ for its posterior
mean and $k_t^\Delta(r,r')$ for the posterior covariance between
$D_rU$ and $D_{r'}U$. 

\paragraph{Induced law of the estimand.}
Substituting \Cref{eq:likelihood_w-u} into Definition~\ref{def:cpte}, the estimand
becomes a functional of the utility,
\begin{equation}
\label{eq:psi-functional}
T_x(u) = \int_{\mathcal{Y}^2} \Phi\Bigl(\frac{D_{x,z} u}{
\lambda}\Bigr) \, \nu_x(dz), \qquad \Psi_x = T_x(u).
\end{equation}
This holds for a fixed $u$. Since $u$ is now a draw from the posterior
$q_t$ of Section~\ref{sec:model}, $\Psi_x$ is random too. Let $U \sim q_t$, 
the \emph{induced estimand} at $x \in \mathcal{X}$ after $t$ queries is the random variable $\Psi_x = T_x(U) \in [0, 1]$, whose law we denote $\Pi_{t,x}$,
and it represents our
belief about $\Psi_x$ after observing $\mathcal{D}_t$.
Because $T_x$ is non-linear, $\Pi_{t,x}$ is not Gaussian and its density
has no closed form. Its mean also differs from the plug-in value,
$\mathbb{E}_{q_t}[\Psi_x] \neq T_x(m_t)$ in general.

\subsection{Linearisation and error control}
\label{sec:linearization}

Definition~\ref{def:cpte-eig} scores a candidate query $c$ by
$\mathbb{I}(\Psi_x; V_c \mid \mathcal{D}_t, c)$. Since $V_c$ depends on
$U$ only through the contrast $D_c U$, this score is determined by the
joint law of $(\Psi_x, D_c U)$. If this pair were Gaussian, the criterion
would reduce to a one-dimensional integral. We therefore approximate
$\Psi_x$ by an affine function of $U$, which makes the pair jointly
Gaussian, and we bound the error this introduces.
Throughout we assume the
posterior kernel $k_t$ is bounded by $K_t$ on its diagonal
(Assumption~\ref{ass:bounded-kernel}, Appendix~\ref{app:integrability}),
which holds for standard kernels such as RBF and Mat\'ern
\citep{rasmussen_Gaussian_2005a} and bounds the variance of every
posterior contrast by $4K_t$. Consequently, $z \mapsto D_{x,z}h$
lies in $L^2(\nu_x) \subset L^1(\nu_x)$ for $q_t$-almost every
$h = U - m_t$ (Lemma~\ref{lem:regularity},
Appendix~\ref{app:integrability}), which allows us to differentiate under
the integral sign below \citep{folland_Real_2013}.

\begin{theorem}[Linearised estimand and its error]
\label{thm:linearisation}
Under Assumption~\ref{ass:bounded-kernel}, write $h = U - m_t$ and, for
$x \in \mathcal X$ and $z \in \mathcal Y^2$, the sensitivity weight
\begin{equation}
\label{eq:sensitivity}
\alpha_t(x,z) = \frac{1}{\lambda} \varphi\Bigl(\frac{m_t^\Delta(x,z)}{\lambda}\Bigr), \qquad \varphi = \Phi',
\end{equation}
which is largest on pairs the model finds hardest to rank. 
Then, the map $\epsilon
\mapsto T_x(m_t+\epsilon h)$ is differentiable at $0$, and its
first-order expansion at $\epsilon=1$ gives
\begin{equation}
\label{eq:linearisation}
\Psi_x \approx T_x(m_t) + \int \alpha_t(x,z)\bigl\{D_{x,z}U - m_t^\Delta(x,z)\bigr\}\,\nu_x(dz),
\end{equation}
which is affine in $U$, hence jointly Gaussian with any contrast $D_cU$. The
remainder $R_x(h)$ neglected in \Cref{eq:linearisation} satisfies
$\mathbb E_{q_t}\bigl[|R_x(h)|\bigr] \;\le\; \frac{\varphi(1)}{2\lambda^2}\int k_t^\Delta\bigl((x,z),(x,z)\bigr)\,\nu_x(dz) \;=:\; B_t(x)$.
\end{theorem}

The proof is in Appendix~\ref{app:linearisation}. From here on, we take
\Cref{eq:linearisation} as the definition of $\Psi_x$: a $\nu_x$-average
of utility contrasts, each weighted by $\alpha_t(x,z)$, that is, by how
much the pair can still move the estimand. The joint law of
$(\Psi_x, D_cU)$ is then
Gaussian and known in closed form, and Section~\ref{sec:closed-form} uses
it to compute CPTE-EIG. Up to the factor $\varphi(1)/(2\lambda^2)$, the
bound $B_t(x)$ is the $\nu_x$-average of the posterior variances of these
contrasts. It is small when the model is certain about the contrasts on
the target pairs, and loose otherwise. Since it depends only on the
current posterior, it can be computed at every round.

\subsection{Closed-form CPTE-EIG}
\label{sec:closed-form}

We now turn the joint law of $(\Psi_x, D_c U)$ into a computable
acquisition score, using the symmetry of mutual information that gives 
$\mathrm{EIG}_{\Psi(x)}(c) = \mathbb{H}[V_c \mid \mathcal{D}_t, c] -
\mathbb{E}_{\Psi_x \sim \Pi_{t,x}}\bigl[\mathbb{H}[V_c \mid \Psi_x,
\mathcal{D}_t, c]\bigr]$.
Expanding on the $V_c$ side, avoids
updating the full posterior for every candidate: since $V_c$ is
binary, each term is the binary entropy $h(\cdot)$ of a success
probability, and the two terms differ only in what that probability
conditions on. The first conditions on what we already know before
answering, giving the marginal $p_{t,c} =
\mathbb{P}(V_c=1\mid\mathcal D_t,c)$: how likely a ``yes'' is
now. The second additionally conditions on
the value of the estimand itself, giving $\eta_{t,x,c}(\psi) =
\mathbb{P}(V_c=1\mid\Psi_x=\psi,\mathcal D_t,c)$: how likely a ``yes'' would be if $\Psi_x=\psi$ were already known. The gap
between $h(p_{t,c})$ and the $\Psi_x$-averaged
$h(\eta_{t,x,c}(\Psi_x))$ is how much knowing $\Psi_x$ would sharpen
our prediction of the answer, which also means by symmetry, how much the answer tells us about $\Psi_x$.

\begin{theorem}[Joint law and closed-form CPTE-EIG]
\label{thm:closed-form}
\label{prop:joint-law}
\label{lem:answer-probabilities}
Fix a target covariate $x\in\mathcal X$ and a candidate query $c$, and
abbreviate $\alpha_t(z)=\alpha_t(x,z)$ and $k_t^\Delta(z,z') =
k_t^\Delta\bigl((x,z),(x,z')\bigr)$. Under $q_t$,
$(\Psi_x, D_cU)$ is Gaussian with mean $\bigl(T_x(m_t),
m_t^\Delta(c)\bigr)$ and covariance $\Sigma = \begin{pmatrix}
\sigma_x^2 & \sigma_{x,c}\\ \sigma_{x,c} & \sigma_c^2\end{pmatrix}$,
where
\begin{align}
\sigma_x^2 &= \int\!\!\int \alpha_t(z)\alpha_t(z')\,k_t^\Delta(z,z')\,\nu_x(dz)\,\nu_x(dz'), \label{eq:sigma-x}\\
\sigma_{x,c} &= \int \alpha_t(z)\,k_t^\Delta(z,c)\,\nu_x(dz), \qquad
\sigma_c^2 = k_t^\Delta(c,c). \label{eq:sigma-c-xc}
\end{align}
Writing $\tau_{x,c}=\sigma_{x,c}/\sigma_x$ and
$s_{x,c}^2 = \lambda^2+\sigma_c^2-\tau_{x,c}^2$, 
then,
\begin{equation}
\label{eq:quadrature}
\mathrm{EIG}_{\Psi(x)}(c) \approx h\left(\Phi\!\left(\frac{m_t^\Delta(c)}{\sqrt{\lambda^2+\sigma_c^2}}\right)\right) - \sum_{i=1}^n
\frac{w_i}{\sqrt\pi}\,h\!\left(\Phi\!\left(\frac{m_t^\Delta(c)+\sqrt2\,\beta_i\,\tau_{x,c}}{s_{x,c}}\right)\right),
\end{equation}
where $(\beta_i,w_i)_{i\le n}$ are Gauss--Hermite nodes and weights.
\end{theorem}

The proof, including the Gauss--Hermite quadrature step
\citep{liu_Note_1994}, is in Appendix~\ref{app:closed-form}. The
covariance $\sigma_{x,c}$ links the query to the
estimand. If $\sigma_{x,c}=0$, then $\tau_{x,c}=0$ and
\Cref{eq:quadrature} is zero, even when $\sigma_c^2>0$: the model is
uncertain about the query, but its answer carries no information about
$\Psi_x$. Example~\ref{ex:cancel} in Appendix~\ref{app:example} gives
such a query.
\Cref{eq:quadrature} requires no posterior sampling and no model refit:
all its terms are computed from $m_t$ and $k_t$ through
\Cref{eq:sigma-x,eq:sigma-c-xc}, at a fixed cost per candidate. The
CPTE-EIG score then averages $\mathrm{EIG}_{\Psi(x)}(c)$ over
${x^\star\sim p^\star}$. We approximate this outer expectation and the
$\nu_x$-integrals by Monte Carlo \citep{robert_Monte_2011}, as detailed
in Algorithm~\ref{alg:cpte-eig} (Appendix~\ref{app:algorithm}).

\section{Experiments}
\label{sec:experiments}

We evaluate CPTE-EIG on the Parkinson's disease example, where dopamine
agonists are compared with levodopa, two standard first-line treatments.
The main quantities are summarised in Table~\ref{tab:parkinson}. We consider four
follow-up horizons $h\in\{1,3,5,7\}$ years, which give four estimands
$\Psi_x^{(h)}$. We compare CPTE-EIG with three baselines: BALD
\citep{houlsby_Bayesian_2011b}, a widely used criterion for preference
learning, the prediction-oriented EPIG
\citep{smith_PredictionOriented_2023b}, and random acquisition.
EPIG is given the same targets as CPTE-EIG, namely the treated--control
pairs drawn from $\nu_x$ for $x\sim p^\star$. The two methods therefore
differ only in the quantity they target: the preference label of each
pair for EPIG, and their average $\Psi_x$ for CPTE-EIG.

\begin{figure}[t]
\centering
\includegraphics[width=\textwidth]{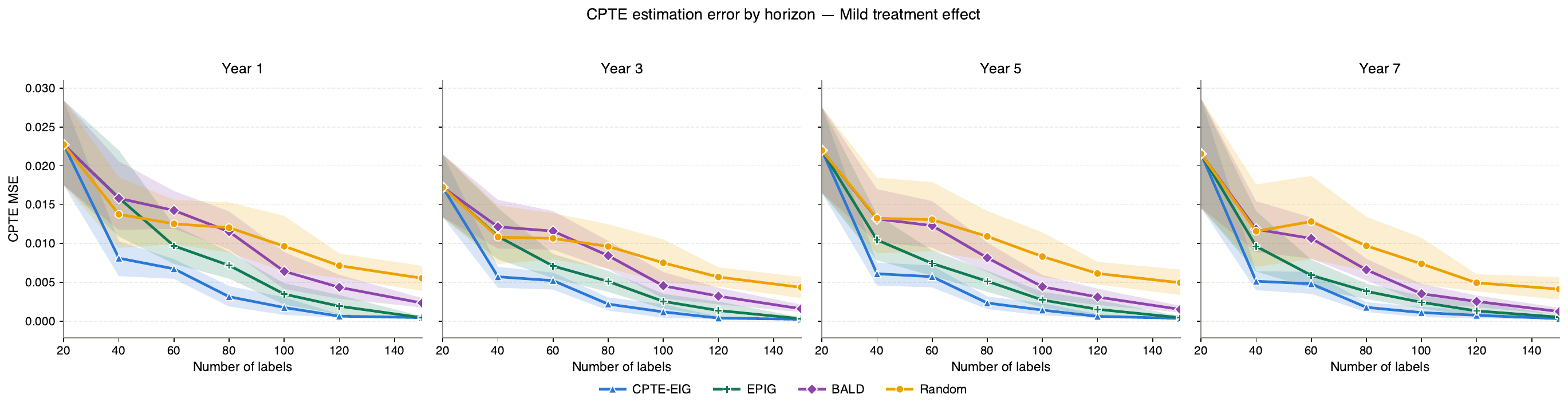}\\
\includegraphics[width=\textwidth]{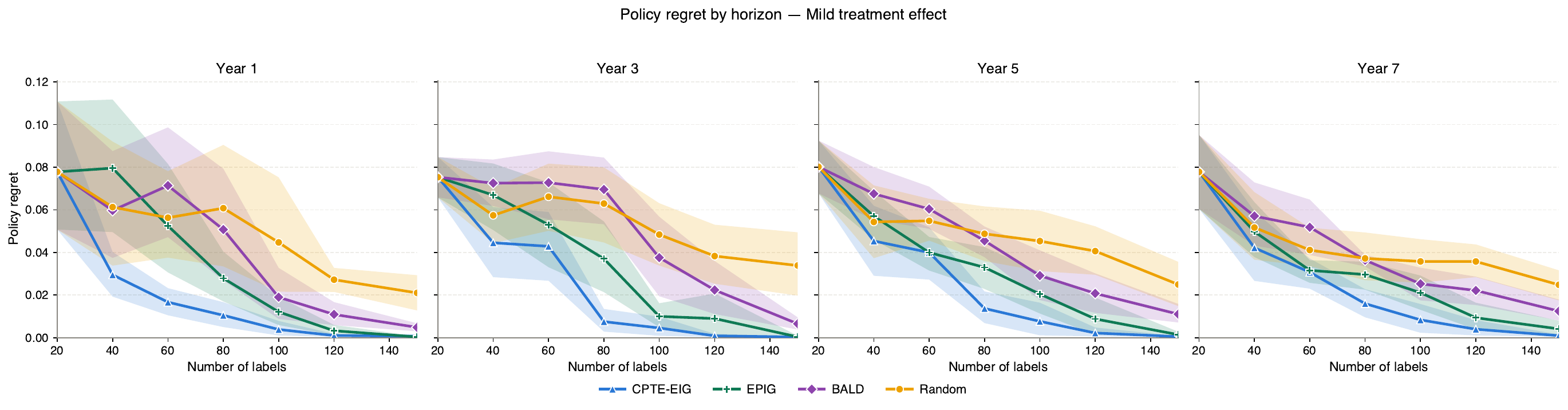}
\caption{Mild treatment effect, per horizon $h$. Top: MSE of $T_x(m_t)$ against the oracle CPTE over 100 target patients from $p^\star$. Bottom: regret of $\pi_t(x)=\mathbf{1}\{T_x(m_t)>1/2\}$. Mean over 15 seeds, 95\% bootstrap intervals. Lower is better.}
\label{fig:effect-cpte}
\end{figure}

\textbf{Setup.}
We build a semi-synthetic cohort of 2,000 patients from PPMI
\citep{marek_Parkinsons_2018} with both potential outcomes known, giving
an oracle CPTE against which every method is scored, and a separate pool
of 1.37M candidate comparisons between real PPMI profiles for
elicitation. The clinician is simulated with a known utility that is
linear in the outcomes, with covariate-dependent weights,
$
u(\tilde x, y) = -\sum_{d} w_d(\tilde x) \sum_{j\in G_d} y_j,
$
where $G_d$ groups the MDS-UPDRS items into five domains: neuropsychiatric,
sleep and autonomic, motor, motor complications, and quality of life.
All weights are positive, so $u$ is decreasing in every symptom. The neuropsychiatric weight increases linearly with age, the motor weight decreases linearly with age, and the
motor-complication weights depend on sex. The other weights are constant. Learning $u$ over the whole covariate space therefore
requires comparisons across all ages and both sexes. We consider three treatment effects. In the \emph{tiny} effect, the
potential outcomes are those of the semi-synthetic cohort, where the two treatments lead to nearly identical outcomes, leading to $\Psi_x^{(h)}\approx
1/2$. In the \emph{mild} and \emph{strong} effects, we keep the original $Y_0$ and set $Y_1=\mathrm{clip}(Y_0+s\,\delta(\tilde x,h),0,4)$, with $s=1$ and $s=2$. Here the agonist improves tremor but worsens cognition by an
amount that grows with $h$ and age, so the best treatment depends on the
patient, the horizon and the utility. 
All methods use the same Gaussian process model for the latent utility,
the same warm start of 20 random comparisons, and the same candidate sets. The kernel hyperparameters
and $\lambda$ are estimated by marginal likelihood at each iteration.
The budget ranges from 20 to 150 comparisons, and each experiment is
repeated with 15 independent seeds. Reported intervals are 95\% bootstrap
intervals for the mean over seeds, computed from 3000 resamples. Further
details on the GP and the setup are given in Appendix~\ref{app:experiments}.

\textbf{Results.}
With a mild treatment effect (Figure~\ref{fig:effect-cpte}), CPTE-EIG
has the lowest CPTE MSE and policy regret at every budget from 40
comparisons on. With 80 comparisons, its error ($2.4\times10^{-3}$) is
already below that of EPIG with 100 comparisons and of BALD with 120, and
its regret ($0.012$) is about a third of EPIG's and a quarter of BALD's.
Since the three criteria share the same utility model, the differences in performance come only from the quantity
they target: the utility as a whole for BALD, the label of each treated--control pair for EPIG, and the
average over these pairs, $\Psi_x$, for CPTE-EIG.
With a strong effect, the gap is larger: at 80 comparisons, the error of
CPTE-EIG is five times lower than that of EPIG and fourteen times lower
than that of BALD (Table~\ref{tab:summary}), and the interval of BALD
does not overlap that of CPTE-EIG at any budget.
With a tiny effect, CPTE-EIG, EPIG and BALD are tied, with overlapping
intervals at every budget, and all three have a lower error than random
acquisition. The errors are also smaller than with larger effects: at 80
comparisons, the error of random acquisition is $2.9$, against $10.6$ and
$24.3$ with mild and strong effects (all $\times10^{-3}$).
Overall, CPTE-EIG is the best-performing method in terms of CPTE accuracy,
policy regret and robustness across scenarios: it is as good as the other
methods when the effect is tiny, and outperforms them when the effect is
mild or strong.

\begin{table}[t]
\centering
\footnotesize
\caption{CPTE MSE ($\times 10^{3}$) against the number of comparisons, over 100 target patients
from $p^\star$ (fixed across seeds and methods) and averaged over horizons: mean over
15 seeds and 95\% percentile-bootstrap interval. Star $^{*}$ for lowest mean (ties included);
\textbf{bold} for methods with interval overlapping that of a $^{*}$ method.}
\label{tab:summary}
\begin{tabular}{llcccc}
\toprule
& & \multicolumn{4}{c}{Budget (number of comparisons)}\\
\cmidrule(lr){3-6}
Effect & Method & 40 & 80 & 120 & 150\\
\midrule
Tiny & Random   & 5.0 [4.0, 6.2] & 2.9 [2.3, 3.5] & 2.2 [1.8, 2.5] & 2.5 [2.1, 3.0]\\
& BALD     & \textbf{$^{*}$3.2 [2.7, 3.7]} & \textbf{1.3 [1.1, 1.4]} & \textbf{$^{*}$0.6 [0.5, 0.7]} & \textbf{$^{*}$0.3 [0.3, 0.4]}\\
& EPIG     & \textbf{3.3 [2.7, 3.9]} & \textbf{$^{*}$1.1 [1.0, 1.2]} & \textbf{0.7 [0.6, 0.8]} & \textbf{0.4 [0.3, 0.5]}\\
& CPTE-EIG & \textbf{$^{*}$3.2 [2.8, 3.7]} & \textbf{1.3 [1.1, 1.4]} & \textbf{$^{*}$0.6 [0.5, 0.7]} & \textbf{0.4 [0.3, 0.5]}\\
\midrule
Mild & Random   & 12.4 [8.4, 17.0] & 10.6 [7.6, 13.7] & 6.0 [4.7, 7.3] & 4.7 [3.3, 6.2]\\
& BALD     & 13.2 [10.2, 16.9] & 8.7 [7.0, 10.6] & 3.3 [2.3, 4.4] & 1.7 [1.2, 2.2]\\
& EPIG     & 11.7 [8.5, 15.9] & 5.3 [4.0, 6.7] & \textbf{1.5 [0.7, 2.8]} & \textbf{0.4 [0.3, 0.6]}\\
& CPTE-EIG & \textbf{$^{*}$6.3 [4.7, 7.7]} & \textbf{$^{*}$2.4 [1.5, 3.3]} & \textbf{$^{*}$0.6 [0.4, 0.9]} & \textbf{$^{*}$0.3 [0.3, 0.5]}\\
\midrule
Strong & Random   & 29.4 [19.3, 41.1] & 24.3 [16.6, 32.7] & 13.7 [10.8, 16.6] & 11.3 [7.6, 15.2]\\
& BALD     & 29.4 [23.2, 36.9] & 18.3 [14.4, 22.3] & 7.0 [4.8, 9.3] & 3.4 [2.4, 4.6]\\
& EPIG     & \textbf{20.5 [13.4, 29.8]} & 6.6 [4.5, 8.9] & 2.5 [1.3, 4.0] & \textbf{0.6 [0.4, 0.7]}\\
& CPTE-EIG & \textbf{$^{*}$13.4 [10.7, 16.1]} & \textbf{$^{*}$1.3 [1.0, 1.6]} & \textbf{$^{*}$0.9 [0.7, 1.2]} & \textbf{$^{*}$0.4 [0.4, 0.5]}\\
\bottomrule
\end{tabular}
\end{table}

\section{Discussion and conclusion}

In this paper, we introduced CPTE-EIG, an active learning criterion that
selects clinician comparisons by their expected information gain about the
CPTE. We linearised the CPTE around the posterior mean utility, which
gave a Gaussian approximation of its posterior law, with an error bound.
Under this approximation, CPTE-EIG becomes computationally tractable.
On a semi-synthetic Parkinson's disease cohort with 13 outcomes, CPTE-EIG
was the best-performing method in terms of CPTE estimation accuracy,
policy regret and robustness across scenarios.

Our framework has limitations that suggest several extensions. We
condition throughout on the estimated potential-outcome laws
$\hat{\mu}_a$ and treat the preference rule as the only source of
uncertainty. Propagating the uncertainty in $\hat{\mu}_a$ as well would
give a more complete account of the uncertainty in the CPTE estimate.
We also model a single, average clinician, whereas real clinicians
can disagree with one another. Multi-user preference models
\citep{houlsby_Collaborative_2012}, for instance with a shared utility and
a per-clinician offset, could represent this disagreement. The choice of
which clinician to ask would then become a second decision, next to the
choice of which pair to show.
Our experiments use 13 outcomes, more than the two to four typically
considered in the composite-endpoint or multi-criteria decision literature
\citep{freemantle_Composite_2003}, and CPTE-EIG performs best in this
setting. With many more outcomes, the limit may come from elicitation rather than
from the model, as profiles with too many dimensions may be hard to compare
in a single judgment.
Finally, the clinician in our experiments is simulated from a known
utility, so that the error can be measured against an oracle CPTE. An
important next step is to test the method on real preference data, which
will include disagreement between clinicians and inconsistent answers.

When a treatment effect is built from clinician preferences, how these
preferences are elicited matters as much as how the effect is estimated.
In our Parkinson's disease experiments, targeting the estimand itself made a limited
elicitation budget go further, and we expect the same to hold wherever
outcomes are compared rather than measured on a common scale.

\subsubsection*{Acknowledgments}
We are grateful to Julie Josse for her regular feedback and advice throughout this project. We thank Tristan Fauvel for helpful discussions on preference learning during his time at Theremia. Finally, we thank Freddie Bickford Smith for kindly taking the time to meet with us and answer our questions.

\subsection*{AI use statement}
We used large language models (LLMs) to help derive our theoretical analysis, and LLM coding assistants to help implement our method and experiments. We also used LLMs to help with the writing and proofreading of the paper, and to give feedback on drafts, including on the methodology and the experiments. The authors verified all AI-assisted derivations, reviewed and tested all AI-assisted code, and checked and edited all AI-assisted text. We take full responsibility for the final content of this work, including text, claims, proofs and code produced with the help of LLMs.

\subsection*{Ethics statement}
Our experiments use data from the Parkinson's Progression Markers Initiative (PPMI) \citep{marek_Parkinsons_2018}, an observational cohort approved by the ethics committees of the participating sites, in which all participants gave written informed consent. We used the de-identified data under the PPMI data use agreement and do not redistribute them. We collected no new data from patients or clinicians: the clinician who answers the comparisons is simulated.

The experiments are designed to evaluate the methodology, not to answer a clinical question. The evaluation cohort is semi-synthetic, and the utilities that play the role of the clinician were chosen by us rather than elicited from clinicians. The estimated treatment effects and policies therefore have no clinical meaning and should not be used to guide the treatment of patients with Parkinson's disease. More generally, a CPTE is only as valid as its inputs. It requires the causal assumptions of Section~\ref{sec:data}, and it depends on whose preferences are elicited and on which patient characteristics enter the query covariate: our choice of $\gamma$, for instance, excludes ethnicity (Section~\ref{sec:preferences}). Since CPTE-EIG concentrates comparisons on the target population, the learned preference rule may also be less accurate for groups that are rare in this population. Clinical use would require preferences elicited from real clinicians, and ideally from patients, and validation on independent data.

\subsection*{Reproducibility statement}
All assumptions are stated in the paper (Assumption~\ref{ass:identification} in Section~\ref{sec:data} and Assumption~\ref{ass:bounded-kernel} in Appendix~\ref{app:integrability}), and complete proofs of Theorems~\ref{thm:linearisation} and~\ref{thm:closed-form} are given in Appendices~\ref{app:integrability}--\ref{app:closed-form}. Algorithm~\ref{alg:cpte-eig} (Appendix~\ref{app:algorithm}) details the query selection procedure. Appendix~\ref{app:experiments} describes the construction of the semi-synthetic cohort and of the candidate pool, the simulated clinicians, the compared methods, the budgets, the seeds and the metrics. 

\bibliography{iclr2027_conference}
\bibliographystyle{iclr2027_conference}


\clearpage

\begin{appendices}

\setcounter{figure}{0}
\renewcommand{\thefigure}{S\arabic{figure}}

\setcounter{table}{0}
\renewcommand{\thetable}{S\arabic{table}}

\startcontents[appendices] 
\printcontents[appendices]{}{1}{}

\section{Bounded contrasts and integrability}
\label{app:integrability}

This appendix states the assumption and regularity result used by
Theorem~\ref{thm:linearisation} in Section~\ref{sec:linearization},
proves the variance bound they rely on, and then proves lemma \ref{lem:regularity}.

\begin{assumption}[Bounded posterior variance]
\label{ass:bounded-kernel}
The posterior kernel $k_t$ is measurable and bounded on its diagonal:
there exists $K_t < \infty$ such that $k_t(r, r) \leq K_t$ for all
$r \in \tilde{\mathcal{X}} \times \mathcal{Y}$.
\end{assumption}

Recall that for two contrast indices $r=(r_1,r_0)$ and $r'=(r_1',r_0')$ in $\tilde{\mathcal X}\times\mathcal Y$,
\[
  k_t^\Delta(r,r') = k_t(r_1,r_1') - k_t(r_1,r_0') - k_t(r_0,r_1') + k_t(r_0,r_0').
\]

Setting $r'=r$ gives the variance, under $q_t$, of a single contrast:
\[
  k_t^\Delta(r,r) = \operatorname{Var}_{q_t}(D_r U)
  = k_t(r_1,r_1) + k_t(r_0,r_0) - 2\,k_t(r_1,r_0),
\]
using the symmetry $k_t(r_1,r_0)=k_t(r_0,r_1)$ of the kernel.

Assumption~\ref{ass:bounded-kernel} bounds the two diagonal terms,
$k_t(r_1,r_1)\le K_t$ and $k_t(r_0,r_0)\le K_t$. Since $k_t$ is a covariance
kernel it is positive semi-definite, so Cauchy--Schwarz applies to the cross
term:
\[
  |k_t(r_1,r_0)| \le \sqrt{k_t(r_1,r_1)\,k_t(r_0,r_0)} \le K_t.
\]
Combining the three bounds,
\begin{equation}
  0 \;\le\; k_t^\Delta(r,r) \;\le\; K_t + K_t + 2K_t \;=\; 4K_t \;<\; \infty,
  \label{eq:contrast-var-bound}
\end{equation}
for every $r\in\tilde{\mathcal X}\times\mathcal Y$. The same argument applies
to $k_t^\Delta$ itself, which is a valid covariance kernel on contrasts
(each $D_rU$ is a linear functional of the Gaussian field $U$), so for two
possibly different contrasts $r,r'$,
\[
  |k_t^\Delta(r,r')| \le \sqrt{k_t^\Delta(r,r)\,k_t^\Delta(r',r')} \le 4K_t.
\]
Every posterior contrast therefore has variance at most $4K_t$, and every
pair of contrasts has covariance at most $4K_t$ in absolute value.

\begin{lemma}[Regularity of the contrast]
\label{lem:regularity}
Under Assumption~\ref{ass:bounded-kernel}, write $h = U - m_t$. For every
$x \in \mathcal{X}$ and $q_t$-almost every realisation of $h$, the map $z
\mapsto D_{x,z} h$ belongs to $L^2(\nu_x) \subset L^1(\nu_x)$.
\end{lemma}

\begin{proof}
Fix $x\in\mathcal X$ and write $h = U - m_t$, a $q_t$-centered Gaussian
process with covariance $k_t$. For $z=(y_1,y_0)\in\mathcal Y^2$, recall that
$D_{x,z}h = h(\gamma(x),y_1) - h(\gamma(x),y_0)$, and define the nonnegative
random variable
\[
  S_x(h) = \int_{\mathcal Y^2} \big(D_{x,z}h\big)^2\,\nu_x(dz).
\]

Since $k_t$ is measurable (Assumption~\ref{ass:bounded-kernel}), the map
$(z,\omega)\mapsto \big(D_{x,z}h(\omega)\big)^2$ is jointly measurable and
nonnegative, so Tonelli's theorem lets us exchange the expectation under
$q_t$ with the integral over $\nu_x$:
\[
\mathbb E_{q_t}[S_x(h)]
=
\mathbb E_{q_t}
\left[
\int (D_{x,z}h)^2\,\nu_x(dz)
\right]
=
\int
\mathbb E_{q_t}\left[(D_{x,z}h)^2\right]
\nu_x(dz).
\]
Because $h$ is centered under $q_t$, so is $D_{x,z}h$, hence
\[
  \mathbb E_{q_t}\big[(D_{x,z}h)^2\big]
  = \operatorname{Var}_{q_t}(D_{x,z}U)
  = k_t^\Delta\big((x,z),(x,z)\big)
  \;{\le}\; 4K_t,
\]
taking $r=(x,z)$, i.e.\ $r_1=(\gamma(x),y_1)$, $r_0=(\gamma(x),y_0)$, in
\Cref{eq:contrast-var-bound}. Since $\nu_x$ is a probability measure,
\[
  \mathbb E_{q_t}[S_x(h)] \;\le\; \int_{\mathcal Y^2} 4K_t\,\nu_x(dz) \;=\; 4K_t \;<\; \infty.
\]

$S_x(h)$ is nonnegative with finite expectation, so it is $q_t$-almost
surely finite: $S_x(h)<\infty$ for $q_t$-almost every $h$. By definition of
$S_x$, this says exactly that, for $q_t$-almost every $h$,
\[
  z \mapsto D_{x,z}h \;\in\; L^2(\nu_x).
\]

Since $\nu_x$ is a probability measure, Cauchy--Schwarz gives, for
$q_t$-almost every $h$,
\[
  \int_{\mathcal Y^2} |D_{x,z}h|\,\nu_x(dz)
  \;\le\;
  \left(\int_{\mathcal Y^2} (D_{x,z}h)^2\,\nu_x(dz)\right)^{1/2}
  = S_x(h)^{1/2} \;<\; \infty,
\]
so $z\mapsto D_{x,z}h$ also belongs to $L^1(\nu_x)$. This holds for every
$x\in\mathcal X$, which proves the lemma.
\end{proof}

\section{Proof of Theorem~\ref{thm:linearisation}}
\label{app:linearisation}

Fix a target covariate $x \in \mathcal X$. The exact CPTE induced by a
posterior draw $U \sim q_t$ is
\[
T_x(U) = \int_{\mathcal Y^2} \Phi\!\left(\frac{D_{x,z}U}{\lambda}\right)\nu_x(dz).
\]
Because of the nonlinear map $\Phi$, $T_x$ is not a linear functional
of $U$. We linearise $T_x$ around the posterior mean $m_t$, bound the
error this introduces, and record the resulting joint Gaussianity of
$(\Psi_x, D_cU)$.

Write $U = m_t + h$, where $h = U - m_t$. By Lemma~\ref{lem:regularity},
for $q_t$-almost every realisation of $h$,
\[
D_{x,\cdot}h \in L^2(\nu_x) \subset L^1(\nu_x).
\]
We fix such a realisation of $h$ for the remainder of the proof.

\subsection{The linearisation}

Following the standard approach to linearising a nonlinear functional,
we study the change of $T_x$ along the one-dimensional path
$m_t + \epsilon h$, $\epsilon \in \mathbb R$. Define
\[
g(\epsilon) = T_x(m_t + \epsilon h).
\]
When it exists, $g'(0)$ is exactly the Gateaux derivative of $T_x$ at
$m_t$ in the direction $h$, $g'(0) = \mathrm dT_x(m_t)(h)$. We later establish its existence directly, by differentiating under the
integral sign, without refering to any general Gateaux-differentiability
theorem. Two values of $g$ matter:
\[
g(0) = T_x(m_t), \qquad g(1) = T_x(m_t+h) = T_x(U).
\]
A first-order Taylor expansion of $g$ around $\epsilon = 0$, evaluated
at $\epsilon = 1$, will therefore give the linearisation of $T_x(U)$
that we want.

By linearity of the contrast operator $D_{x,z}$,
\[
D_{x,z}(m_t + \epsilon h) = m_t^\Delta(x,z) + \epsilon\, D_{x,z}h,
\qquad m_t^\Delta(x,z) := D_{x,z}m_t,
\]
so that
\begin{equation}
\label{eq:g-integral-form}
g(\epsilon) = \int_{\mathcal Y^2} f(\epsilon, z)\,\nu_x(dz),
\qquad
f(\epsilon,z) = \Phi\!\left(\frac{m_t^\Delta(x,z) + \epsilon\, D_{x,z}h}{\lambda}\right).
\end{equation}

We justify that $g'(0)$ exists and can be computed by differentiating
\Cref{eq:g-integral-form} under the integral sign
\citep{folland_Real_2013}. This standard result requires three
conditions to hold for $f(\epsilon,z)$ on some open interval $I$ including $0$:
\begin{itemize}
\item[(i)] for every $\epsilon \in I$, the map $z \mapsto f(\epsilon,z)$
is $\nu_x$-integrable;
\item[(ii)] for $\nu_x$-almost every $z$, the map
$\epsilon \mapsto f(\epsilon,z)$ is differentiable on $I$;
\item[(iii)] there exists $M \in L^1(\nu_x)$ such that, for every
$\epsilon \in I$ and $\nu_x$-almost every $z$,
$|\partial_\epsilon f(\epsilon,z)| \le M(z)$.
\end{itemize}
If (i)--(iii) hold, then $g$ is differentiable on $I$ and
$g'(\epsilon) = \int \partial_\epsilon f(\epsilon,z)\,\nu_x(dz)$.

\textbf{Condition (i).} Immediate, since $0 \le \Phi \le 1$, so
$f(\epsilon, \cdot)$ is bounded and $\nu_x$ is a probability measure.

\textbf{Condition (ii).} Immediate, since $\Phi$ is continuously
differentiable ($\Phi \in C^1$), with derivative $\varphi = \Phi'$,
the standard normal density. Explicitly,
\begin{equation}
\label{eq:partial-eps-f}
\partial_\epsilon f(\epsilon,z) = \frac{D_{x,z}h}{\lambda}\,
\varphi\!\left(\frac{m_t^\Delta(x,z) + \epsilon\, D_{x,z}h}{\lambda}\right).
\end{equation}

\textbf{Condition (iii).} Since $\varphi$ is bounded,
$0 \le \varphi(u) \le 1/\sqrt{2\pi}$ for every $u \in \mathbb R$, we get,
uniformly in $\epsilon$,
\[
|\partial_\epsilon f(\epsilon,z)| \le \frac{|D_{x,z}h|}{\lambda\sqrt{2\pi}} =: M(z).
\]
By Lemma~\ref{lem:regularity}, $D_{x,\cdot}h \in L^1(\nu_x)$, so
$M \in L^1(\nu_x)$: condition (iii) holds.

All three conditions hold, so differentiation under the integral sign
applies. Evaluating \Cref{eq:partial-eps-f} at $\epsilon = 0$ gives
\[
g'(0) = \int_{\mathcal Y^2} \frac{D_{x,z}h}{\lambda}\,
\varphi\!\left(\frac{m_t^\Delta(x,z)}{\lambda}\right)\nu_x(dz)
= \int_{\mathcal Y^2} \alpha_t(x,z)\, D_{x,z}h \,\nu_x(dz),
\]
where we recognise the sensitivity weight $\alpha_t(x,z)$ of
\Cref{eq:sensitivity}.

By Taylor's theorem, $g(1) = g(0) + g'(0) + R_x(h)$, where $R_x(h)$
is the remainder. Substituting $g(0) = T_x(m_t)$, $g(1) = T_x(U)$, and
the expression for $g'(0)$ from Step 2,
\[
T_x(U) = T_x(m_t) + \int_{\mathcal Y^2} \alpha_t(x,z)\, D_{x,z}h\,\nu_x(dz) + R_x(h).
\]
The linearised approximation consists in neglecting $R_x(h)$. Since
$D_{x,z}h = D_{x,z}U - m_t^\Delta(x,z)$, this gives
\begin{equation}
\label{eq:linearisation-proof}
T_x(U) \approx T_x(m_t) + \int_{\mathcal Y^2} \alpha_t(x,z)\,\bigl\{D_{x,z}U - m_t^\Delta(x,z)\bigr\}\,\nu_x(dz),
\end{equation}
which is exactly \Cref{eq:linearisation}. We adopt the right-hand
side of \Cref{eq:linearisation-proof} as the definition of $\Psi_x$
from here on.

The right-hand side of \Cref{eq:linearisation-proof} is an affine functional of the Gaussian process $U$, so $\Psi_x$ is Gaussian under $q_t$. For any
candidate query $c$, the contrast $D_cU$ is also a linear functional
of $U$. Two jointly Gaussian-process-linear functionals of the same
Gaussian field are themselves jointly Gaussian, so $(\Psi_x, D_cU)$ is
jointly Gaussian, as used in Theorem~\ref{thm:closed-form}.

\subsection{The remainder}

We control $R_x(h)$ using Taylor's theorem with integral remainder:
if $g'$ is absolutely continuous on $[0,1]$ and $g'' \in L^1([0,1])$,
then
\begin{equation}
\label{eq:taylor-remainder}
g(1) = g(0) + g'(0) + R_x(h), \qquad R_x(h) = \int_0^1 (1-\epsilon)\,g''(\epsilon)\,d\epsilon.
\end{equation}
We verify the hypothesis by differentiating $\partial_\epsilon f$
under the integral sign a second time, checking conditions (i)--(iii)
again.

Since $\Phi \in C^2$, the second derivative of the integrand is
\begin{equation}
\label{eq:partial2-eps-f}
\partial_\epsilon^2 f(\epsilon,z) = \frac{(D_{x,z}h)^2}{\lambda^2}\,
\Phi''\!\left(\frac{m_t^\Delta(x,z) + \epsilon\, D_{x,z}h}{\lambda}\right).
\end{equation}
\begin{itemize}
\item[(i)] For every $\epsilon$, $z \mapsto \partial_\epsilon f(\epsilon,z)$
is $\nu_x$-integrable, since it is bounded in absolute value by a
constant times $|D_{x,z}h| \in L^1(\nu_x)$ (already used above).
\item[(ii)] For $\nu_x$-almost every $z$,
$\epsilon \mapsto \partial_\epsilon f(\epsilon,z)$ is differentiable,
since $\Phi \in C^2$.
\item[(iii)] We need a dominating $L^1(\nu_x)$ function for
\Cref{eq:partial2-eps-f}. Using $\Phi''(u) = \varphi'(u) = -u\,\varphi(u)$,
\[
|\Phi''(u)| \le \sup_{v \in \mathbb R} |v\,\varphi(v)| = \varphi(1),
\]
so, uniformly in $\epsilon$,
\[
|\partial_\epsilon^2 f(\epsilon,z)| \le \frac{\varphi(1)}{\lambda^2}(D_{x,z}h)^2.
\]
By Lemma~\ref{lem:regularity}, $D_{x,\cdot}h \in L^2(\nu_x)$, so the
right-hand side lies in $L^1(\nu_x)$: condition (iii) holds.
\end{itemize}

All three conditions hold, so $g$ is twice differentiable on $[0,1]$
with
\[
g''(\epsilon) = \int_{\mathcal Y^2} \frac{(D_{x,z}h)^2}{\lambda^2}\,
\Phi''\!\left(\frac{m_t^\Delta(x,z) + \epsilon\, D_{x,z}h}{\lambda}\right)\nu_x(dz),
\]
and, using the bound established in condition (iii) above, for every
$\epsilon \in [0,1]$,
\[
|g''(\epsilon)| \le \frac{\varphi(1)}{\lambda^2}\int_{\mathcal Y^2}(D_{x,z}h)^2\,\nu_x(dz).
\]

Combining this bound with \Cref{eq:taylor-remainder},
\[
|R_x(h)| \le \int_0^1 (1-\epsilon)\,|g''(\epsilon)|\,d\epsilon
\le \frac{\varphi(1)}{\lambda^2}\left(\int_0^1 (1-\epsilon)\,d\epsilon\right)\int_{\mathcal Y^2}(D_{x,z}h)^2\,\nu_x(dz).
\]
Since $\int_0^1(1-\epsilon)\,d\epsilon = 1/2$,
\begin{equation}
\label{eq:pointwise-remainder-bound}
|R_x(h)| \le \frac{\varphi(1)}{2\lambda^2}\int_{\mathcal Y^2}(D_{x,z}h)^2\,\nu_x(dz).
\end{equation}
Taking expectation under $q_t$ on both sides of
\Cref{eq:pointwise-remainder-bound}, and exchanging expectation and
$\nu_x$-integral by Tonelli's theorem exactly as in the proof of
Lemma~\ref{lem:regularity} (the integrand is nonnegative), we get
\[
\mathbb E_{q_t}\bigl[|R_x(h)|\bigr]
\le \frac{\varphi(1)}{2\lambda^2}\int_{\mathcal Y^2} \mathbb E_{q_t}\bigl[(D_{x,z}h)^2\bigr]\,\nu_x(dz).
\]
Since $h$ is centered under $q_t$, so is $D_{x,z}h$, and hence
$\mathbb E_{q_t}[(D_{x,z}h)^2] = \operatorname{Var}_{q_t}(D_{x,z}U) = k_t^\Delta\bigl((x,z),(x,z)\bigr)$.
Substituting,
\[
\mathbb E_{q_t}\bigl[|R_x(h)|\bigr]
\le \frac{\varphi(1)}{2\lambda^2}\int_{\mathcal Y^2} k_t^\Delta\bigl((x,z),(x,z)\bigr)\,\nu_x(dz) =: B_t(x),
\]
which is the bound from the theorem. This completes the proof.

\medskip
\noindent\textit{Interpretation.} $B_t(x)$ is the average, over
$z \sim \nu_x$, of the posterior variance of the contrast
$D_{x,z}U$. The linearisation is therefore accurate when the
posterior is already confident about how outcome pairs at $x$
compare (small $k_t^\Delta$ relative to $\lambda^2$), and can be loose
early in active learning, when the posterior is still wide, but
tightens as informative comparisons accumulate.

\section{Proof of Theorem~\ref{thm:closed-form}}
\label{app:closed-form}

Throughout, fix a target covariate $x \in \mathcal X$ and a candidate
query $c$, and use the shorthand $\alpha_t(z) = \alpha_t(x,z)$ and
$k_t^\Delta(z,z') = k_t^\Delta\bigl((x,z),(x,z')\bigr)$ from the
theorem statement. The proof has four parts: (1) the mean and
covariance of $(\Psi_x, D_cU)$; (2) two standard properties about jointly
Gaussian and probit-linked random variables; (3) the closed-form
answer probabilities $\eta_{t,x,c}$ and $p_{t,c}$, obtained by
applying these properties; and (4) the Gauss--Hermite quadrature that
turns the resulting expectation into \Cref{eq:quadrature}.

\subsection{Mean and covariance of $(\Psi_x, D_cU)$}
\label{app:closed-form-covariance}

By Theorem~\ref{thm:linearisation},
\[
\Psi_x = T_x(m_t) + \int_{\mathcal Y^2} \alpha_t(z)\,\bigl\{D_{x,z}U - m_t^\Delta(x,z)\bigr\}\,\nu_x(dz),
\]
an affine functional of $U$, and $D_cU$ is a linear functional of $U$.
Since $U$ is Gaussian under $q_t$, so is any finite collection of
linear (or affine) functionals of $U$. In particular $(\Psi_x, D_cU)$
is jointly Gaussian.

\paragraph{Mean} Taking expectations, $\mathbb E_{q_t}[D_{x,z}U -
m_t^\Delta(x,z)] = 0$ for every $z$ (by definition of $m_t^\Delta$ as
the posterior mean contrast), so $\mathbb E_{q_t}[\Psi_x] = T_x(m_t)$.
Similarly $\mathbb E_{q_t}[D_cU] = D_c m_t = m_t^\Delta(c)$. The mean
vector is therefore $\bigl(T_x(m_t), m_t^\Delta(c)\bigr)$.

\paragraph{Variance $\sigma_x^2$} Write $G(z) = D_{x,z}U$ and
$\widetilde G(z) = G(z) - \mathbb E_{q_t}[G(z)]$, so that
\[
\Psi_x - \mathbb E_{q_t}[\Psi_x] = \int_{\mathcal Y^2} \alpha_t(z)\,\widetilde G(z)\,\nu_x(dz).
\]
Then
\begin{equation*}
\begin{aligned}
\sigma_x^2 = \operatorname{Var}_{q_t}(\Psi_x)
&= \mathbb E_{q_t}\!\left[\left(\int \alpha_t(z)\,\widetilde G(z)\,\nu_x(dz)\right)\!\left(\int \alpha_t(z')\,\widetilde G(z')\,\nu_x(dz')\right)\right] \\
&= \mathbb E_{q_t}\!\left[\iint \alpha_t(z)\alpha_t(z')\,\widetilde G(z)\widetilde G(z')\,\nu_x(dz)\,\nu_x(dz')\right].
\end{aligned}
\end{equation*}
We exchange $\mathbb E_{q_t}$ and the double $\nu_x$-integral by
Fubini's theorem. This is justified because the integrand is
absolutely integrable with respect to $q_t \otimes \nu_x \otimes
\nu_x$:
\begin{itemize}
\item $\alpha_t$ is uniformly bounded, since $\alpha_t(z) =
\varphi\bigl(m_t^\Delta(x,z)/\lambda\bigr)/\lambda \le 1/(\lambda\sqrt{2\pi})
=: \bar\alpha$ for every $z$;
\item by Cauchy--Schwarz in $L^2(q_t)$ and the bound
\Cref{eq:contrast-var-bound},
\[
\mathbb E_{q_t}\bigl[|\widetilde G(z)\widetilde G(z')|\bigr]
\le \sqrt{\operatorname{Var}_{q_t}(G(z))\,\operatorname{Var}_{q_t}(G(z'))}
= \sqrt{k_t^\Delta(z,z)\,k_t^\Delta(z',z')} \le 4K_t;
\]
\item $\nu_x \otimes \nu_x$ is a probability measure on $\mathcal
Y^2 \times \mathcal Y^2$.
\end{itemize}
The integrand is therefore bounded in absolute value by the constant
$\bar\alpha^2 \cdot 4K_t$, which is integrable against the finite
measure $q_t \otimes \nu_x \otimes \nu_x$. Fubini's theorem applies,
so we may exchange the order of integration:
\[
\sigma_x^2 = \iint \alpha_t(z)\alpha_t(z')\,\mathbb E_{q_t}\bigl[\widetilde G(z)\widetilde G(z')\bigr]\,\nu_x(dz)\,\nu_x(dz')
= \iint \alpha_t(z)\alpha_t(z')\,\operatorname{Cov}_{q_t}\bigl(D_{x,z}U, D_{x,z'}U\bigr)\,\nu_x(dz)\,\nu_x(dz').
\]
By definition of $k_t^\Delta$ ,
$\operatorname{Cov}_{q_t}(D_{x,z}U, D_{x,z'}U) = k_t^\Delta(z,z')$.

\paragraph{Covariance $\sigma_{x,c}$.} By the same argument,
\[
\sigma_{x,c} = \operatorname{Cov}_{q_t}(\Psi_x, D_cU)
= \operatorname{Cov}_{q_t}\!\left(\int \alpha_t(z)\,D_{x,z}U\,\nu_x(dz),\; D_cU\right).
\]
The exchange of $\operatorname{Cov}_{q_t}$ and the $\nu_x$-integral
is justified as above, now for a single integral: $|\alpha_t(z)| \le
\bar\alpha$ and $|\operatorname{Cov}_{q_t}(D_{x,z}U, D_cU)| \le
\sqrt{k_t^\Delta(z,z)\,\sigma_c^2} \le \sqrt{4K_t \sigma_c^2}$ by
Cauchy--Schwarz, a constant, and $\nu_x$ is a probability measure.
Hence
\[
\sigma_{x,c} = \int \alpha_t(z)\,\operatorname{Cov}_{q_t}(D_{x,z}U, D_cU)\,\nu_x(dz)
= \int \alpha_t(z)\,k_t^\Delta(z,c)\,\nu_x(dz),
\].

\paragraph{Variance $\sigma_c^2$} Immediate from the definition of
$k_t^\Delta$: $\sigma_c^2 = \operatorname{Var}_{q_t}(D_cU) =
k_t^\Delta(c,c)$.

This establishes the covariance matrix $\Sigma$.

\subsection{Two properties}
\label{app:closed-form-lemmas}

\begin{lemma}[Gaussian conditioning]
\label{lem:gaussian-conditioning}
Let $(A,B)$ be jointly Gaussian with means $\mu_A, \mu_B$, variances
$\sigma_A^2, \sigma_B^2 > 0$, and covariance $\sigma_{AB}$. Then
\[
B \mid A = a \;\sim\; \mathcal N\!\left(\mu_B + \frac{\sigma_{AB}}{\sigma_A^2}(a - \mu_A),\;\; \sigma_B^2 - \frac{\sigma_{AB}^2}{\sigma_A^2}\right).
\]
\end{lemma}

\begin{proof}
Define $B' = B - \frac{\sigma_{AB}}{\sigma_A^2}(A - \mu_A)$, a linear
combination of $A$ and $B$, hence jointly Gaussian with $A$. Its
covariance with $A$ is
\[
\operatorname{Cov}(A, B') = \sigma_{AB} - \frac{\sigma_{AB}}{\sigma_A^2}\operatorname{Var}(A)
= \sigma_{AB} - \sigma_{AB} = 0.
\]
For jointly Gaussian random variables, zero covariance implies
independence, so $B' \perp A$. Its mean and variance are
\[
\mathbb E[B'] = \mu_B, \qquad
\operatorname{Var}(B') = \sigma_B^2 - 2\frac{\sigma_{AB}}{\sigma_A^2}\sigma_{AB} + \frac{\sigma_{AB}^2}{\sigma_A^4}\sigma_A^2
= \sigma_B^2 - \frac{\sigma_{AB}^2}{\sigma_A^2},
\]
so $B' \sim \mathcal N\bigl(\mu_B,\, \sigma_B^2 - \sigma_{AB}^2/\sigma_A^2\bigr)$,
independently of $A$. Since $B = B' + \frac{\sigma_{AB}}{\sigma_A^2}(A - \mu_A)$,
conditioning on $A = a$ leaves $B'$'s law unchanged (by independence)
and shifts it by the deterministic term $\frac{\sigma_{AB}}{\sigma_A^2}(a-\mu_A)$,
giving the stated conditional law.
\end{proof}

\begin{lemma}[Probit--Gaussian integral identity]
\label{lem:probit-gaussian}
For $W \sim \mathcal N(\mu, s^2)$ and $\lambda > 0$,
\[
\mathbb E\bigl[\Phi(W/\lambda)\bigr] = \Phi\!\left(\frac{\mu}{\sqrt{\lambda^2 + s^2}}\right).
\]
\end{lemma}

\begin{proof}
Let $Y \sim \mathcal N(0,1)$ be independent of $W$. Since $\Phi$ is
the standard normal cdf, $\Phi(w/\lambda) = \mathbb P(Y \le w/\lambda)
= \mathbb P(\lambda Y \le w)$ for any fixed $w$, using $\lambda > 0$.
Hence, conditioning on $W$ and using the tower property,
\[
\mathbb E\bigl[\Phi(W/\lambda)\bigr] = \mathbb E\bigl[\mathbb P(\lambda Y \le W \mid W)\bigr] = \mathbb P(\lambda Y \le W) = \mathbb P(W - \lambda Y \ge 0).
\]
Since $W$ and $Y$ are independent Gaussians, $W - \lambda Y \sim
\mathcal N(\mu, s^2 + \lambda^2)$. Standardising,
\[
\mathbb P(W - \lambda Y \ge 0) = \mathbb P\!\left(Z \ge \frac{-\mu}{\sqrt{s^2+\lambda^2}}\right) = 1 - \Phi\!\left(\frac{-\mu}{\sqrt{s^2+\lambda^2}}\right) = \Phi\!\left(\frac{\mu}{\sqrt{\lambda^2+s^2}}\right),
\]
for $Z \sim \mathcal N(0,1)$, using the antisymmetry $1 - \Phi(-t) = \Phi(t)$.
\end{proof}

\subsection{Closed-form answer probabilities}
\label{app:closed-form-probabilities}

Recall $V_c \mid U \sim \mathrm{Bernoulli}\bigl(\Phi(D_cU/\lambda)\bigr)$
from \Cref{eq:likelihood_w-u}, so $V_c$'s conditional law given $U$ depends on
$U$ only through $D_cU$. Since $\Psi_x$ is a deterministic (affine)
function of $U$, $\sigma(\Psi_x) \subset \sigma(U)$, and the tower
property gives, for any $\psi$,
\[
\eta_{t,x,c}(\psi) = \mathbb P(V_c = 1 \mid \Psi_x = \psi)
= \mathbb E\bigl[\mathbb P(V_c = 1 \mid U) \mid \Psi_x = \psi\bigr]
= \mathbb E\bigl[\Phi(D_cU/\lambda) \mid \Psi_x = \psi\bigr].
\]

\paragraph{Computing $\eta_{t,x,c}(\psi)$} By \ref{app:closed-form},
$(\Psi_x, D_cU)$ is jointly Gaussian with the mean and covariance
derived above. Applying Lemma~\ref{lem:gaussian-conditioning} with
$A = \Psi_x$, $B = D_cU$,
\[
D_cU \mid \Psi_x = \psi \;\sim\; \mathcal N\!\left(m_t^\Delta(c) + \frac{\sigma_{x,c}}{\sigma_x^2}\bigl(\psi - T_x(m_t)\bigr),\;\; \sigma_c^2 - \frac{\sigma_{x,c}^2}{\sigma_x^2}\right).
\]
Writing $\tau_{x,c} = \sigma_{x,c}/\sigma_x$, so that
$\sigma_{x,c}/\sigma_x^2 = \tau_{x,c}/\sigma_x$ and
$\sigma_{x,c}^2/\sigma_x^2 = \tau_{x,c}^2$, this conditional law is
$\mathcal N\bigl(\mu_\psi, s^2\bigr)$ with
\[
\mu_\psi = m_t^\Delta(c) + \frac{\tau_{x,c}}{\sigma_x}\bigl(\psi - T_x(m_t)\bigr),
\qquad s^2 = \sigma_c^2 - \tau_{x,c}^2.
\]
Applying Lemma~\ref{lem:probit-gaussian} conditionally, with $W =
D_cU \mid \Psi_x = \psi$,
\begin{equation}
\begin{aligned} \label{eq:eta}
\eta_{t,x,c}(\psi) &= \Phi\!\left(\frac{\mu_\psi}{\sqrt{\lambda^2 + s^2}}\right)
= \Phi\!\left(\frac{m_t^\Delta(c) + \tau_{x,c}(\psi - T_x(m_t))/\sigma_x}{\sqrt{\lambda^2 + \sigma_c^2 - \tau_{x,c}^2}}\right) \\
&= \Phi\!\left(\frac{m_t^\Delta(c) + \tau_{x,c}(\psi - T_x(m_t))/\sigma_x}{s_{x,c}}\right),
\end{aligned}
\end{equation}
where $s_{x,c}^2 := \lambda^2 + \sigma_c^2 - \tau_{x,c}^2$.

\paragraph{Computing $p_{t,c}$} By \ref{app:closed-form}, the marginal law of
$D_cU$ is $\mathcal N\bigl(m_t^\Delta(c), \sigma_c^2\bigr)$. Since
$V_c \mid U$ depends on $U$ only through $D_cU$,
\[
p_{t,c} = \mathbb P(V_c = 1) = \mathbb E\bigl[\Phi(D_cU/\lambda)\bigr].
\]
Applying Lemma~\ref{lem:probit-gaussian} with $W = D_cU$, $\mu =
m_t^\Delta(c)$, $s^2 = \sigma_c^2$,
\begin{equation}
\label{eq:p_tc}
p_{t,c} = \Phi\!\left(\frac{m_t^\Delta(c)}{\sqrt{\lambda^2 + \sigma_c^2}}\right).
\end{equation}

\subsection{Gauss--Hermite quadrature and the final closed form}
\label{app:closed-form-quadrature}

Remind that
\begin{equation}
\label{eq:eig-answer-side}
\mathrm{EIG}_{\Psi(x)}(c) = \mathbb{H}[V_c \mid \mathcal{D}_t, c] -
\mathbb{E}_{\Psi_x \sim \Pi_{t,x}}\bigl[\mathbb{H}[V_c \mid \Psi_x,
\mathcal{D}_t, c]\bigr].
\end{equation}

By \Cref{eq:eig-answer-side} and the reduction to binary entropies,
\[
\mathrm{EIG}_{\Psi(x)}(c) = h(p_{t,c}) - \mathbb E_{\Psi_x \sim \Pi_{t,x}}\bigl[h(\eta_{t,x,c}(\Psi_x))\bigr].
\]
Under the linearisation of Theorem~\ref{thm:linearisation}, $\Psi_x$
is Gaussian under $q_t$ with mean $T_x(m_t)$ and variance $\sigma_x^2$, so $\Pi_{t,x} = \mathcal N(T_x(m_t), \sigma_x^2)$ has density
\[
f_{\Psi_x}(\psi) = \frac{1}{\sigma_x}\,\varphi\!\left(\frac{\psi - T_x(m_t)}{\sigma_x}\right), \qquad \varphi(u) = \frac{1}{\sqrt{2\pi}}e^{-u^2/2},
\]
so that
\[
\mathbb E_{\Psi_x \sim \Pi_{t,x}}\bigl[h(\eta_{t,x,c}(\Psi_x))\bigr]
= \int_{-\infty}^{\infty} h(\eta_{t,x,c}(\psi))\,\frac{1}{\sigma_x}\,\varphi\!\left(\frac{\psi - T_x(m_t)}{\sigma_x}\right)\,d\psi.
\]

Substitute $\beta = \frac{1}{\sqrt2}\cdot\frac{\psi - T_x(m_t)}{\sigma_x}$, i.e.\ $\psi = T_x(m_t) + \sqrt2\,\beta\,\sigma_x$ and $d\psi = \sqrt2\,\sigma_x\,d\beta$.
Then $\frac{1}{\sigma_x}\varphi\bigl(\tfrac{\psi-T_x(m_t)}{\sigma_x}\bigr)\,d\psi
= \frac{1}{\sigma_x}\cdot\frac{1}{\sqrt{2\pi}}e^{-\beta^2}\cdot\sqrt2\,\sigma_x\,d\beta
= \frac{1}{\sqrt\pi}\,e^{-\beta^2}\,d\beta$, so
\[
\mathbb E_{\Psi_x \sim \Pi_{t,x}}\bigl[h(\eta_{t,x,c}(\Psi_x))\bigr]
= \int_{-\infty}^{\infty} h\bigl(\eta_{t,x,c}(T_x(m_t) + \sqrt2\,\beta\,\sigma_x)\bigr)\,\frac{1}{\sqrt\pi}\,e^{-\beta^2}\,d\beta.
\]

Gauss--Hermite quadrature approximates integrals of the form $\int_{-\infty}^{\infty} e^{-\beta^2} g(\beta)\,d\beta
\approx \sum_{i=1}^n w_i\, g(\beta_i)$, where $(\beta_i)_{i \le n}$ are
the roots of the degree-$n$ Hermite polynomial and $(w_i)_{i \le n}$
the corresponding weights \citep{liu_Note_1994}. Applying this with
$g(\beta) = h\bigl(\eta_{t,x,c}(T_x(m_t) + \sqrt2\,\beta\,\sigma_x)\bigr)$,
\[
\mathbb E_{\Psi_x \sim \Pi_{t,x}}\bigl[h(\eta_{t,x,c}(\Psi_x))\bigr]
\approx \sum_{i=1}^n \frac{w_i}{\sqrt\pi}\,h\bigl(\eta_{t,x,c}(T_x(m_t) + \sqrt2\,\beta_i\,\sigma_x)\bigr).
\]

From \Cref{eq:eta}, at $\psi = T_x(m_t) + \sqrt2\,\beta_i\,\sigma_x$, the
term $\bigl(\psi - T_x(m_t)\bigr)/\sigma_x$ simplifies to
$\sqrt2\,\beta_i$, so
\[
\eta_{t,x,c}\bigl(T_x(m_t) + \sqrt2\,\beta_i\,\sigma_x\bigr) = \Phi\!\left(\frac{m_t^\Delta(c) + \sqrt2\,\beta_i\,\tau_{x,c}}{s_{x,c}}\right).
\]

Substituting back,
\[
\mathrm{EIG}_{\Psi(x)}(c) \approx h(p_{t,c}) - \sum_{i=1}^n \frac{w_i}{\sqrt\pi}\,h\!\left(\Phi\!\left(\frac{m_t^\Delta(c) + \sqrt2\,\beta_i\,\tau_{x,c}}{s_{x,c}}\right)\right),
\]
which is \Cref{eq:quadrature}.


\section{Targeting the right information}
\label{app:example}

By Equation \ref{eq:linearisation}, the linearised estimand is $\Psi_x = c_{t,x} + Z_x$, where $c_{t,x}$ is a
constant given $D_t$ and $Z_x = \int \alpha_t(x,z)\,D_{x,z}U\,\nu_x(dz)$ is a single linear
functional of $U$, the $\alpha_t$-weighted average over the target pool of the pairwise
contrasts. Since $\Psi_x$ is a deterministic shift of $Z_x$,
$\mathrm{EIG}_{\Psi(x)}(c) = I(Z_x; V_c \mid D_t, c)$ for every $c$: CPTE-EIG only sees what
a query reveals about this one averaged contrast, never about any individual pair entering
it. A query can be informative about every pair in the pool taken separately, which is
exactly what EPIG rewards, and still be uninformative about their average. The following
example gives such a query, and contrasts it with both EPIG and BALD.

\begin{example}[Target the right information]
\label{ex:cancel}
Let $y=(y^{(1)},y^{(2)})$ be the outcome, where $y^{(1)}$ records mobility and $y^{(2)}$
records pain relief. Model the utility as linear in $y$, $u(\tilde x,y)=\theta_1
y^{(1)}+\theta_2 y^{(2)}$, with $\theta\sim\mathcal N(0,\sigma^2 I)$ under $q_t$: an instance
of the Gaussian process model of Section \ref{sec:model}, using a linear kernel and a posterior mean of
zero. Contrasts are then plain dot products, $\mathrm{Cov}(\theta\cdot d,\theta\cdot
d')=\sigma^2\,d\cdot d'$, for any two directions $d,d'$.

Fix a target $x$ with $\nu_x=\tfrac12\delta_{z_1}+\tfrac12\delta_{z_2}$, two pairs sharing
the query covariate: at $z_1$ the treated outcome gains on both dimensions, with contrast
$d_1=(1,1)$; at $z_2$ it trades mobility for pain relief, with contrast $d_2=(1,-1)$.
Consider a query $c$ with contrast $d_c=(0,1)$, turning entirely on pain relief. Since
$m_t=0$, every posterior mean contrast vanishes, and by Equation \ref{eq:sensitivity} $z_1$ and $z_2$
receive the same sensitivity weight $\alpha_0=\varphi(0)/\lambda$.

The contrast the target pool averages over is $\tfrac12(d_1+d_2)=(1,0)$: the $\theta_2$
component of $z_1$ and $z_2$ cancel, since the two pairs trade mobility against pain relief
in opposite directions. Hence $\sigma_{x,c}=\alpha_0\sigma^2\,(1,0)\cdot(0,1)=0$, while
$\sigma_x^2=\alpha_0^2\sigma^2>0$ and $\sigma_c^2=\sigma^2>0$, so
$\tau_{x,c}=\sigma_{x,c}/\sigma_x=0$. By Equation \ref{eq:eta}, $\eta_{t,x,c}(\psi)=\Phi(0)$ for
every $\psi$, equal to $p_{t,c}$: the answer to $c$ tells us nothing about $\Psi_x$, and
$\mathrm{EIG}_{\Psi(x)}(c)=h(p_{t,c})-h(p_{t,c})=0$ by Equation \ref{eq:eig-answer-side}.

EPIG does not cancel this way, since it scores $z_1$ and $z_2$ separately rather than
through their average: $\mathrm{Cov}(\theta\cdot d_1,\theta\cdot d_c)=\sigma^2$ and
$\mathrm{Cov}(\theta\cdot d_2,\theta\cdot d_c)=-\sigma^2$, both nonzero, so $c$ is
informative about the answer at $z_1$ and at $z_2$ on their own, and $\mathrm{EPIG}(c)>0$.
BALD is positive as well, since $V_c$ depends on $\theta_2$, which has positive prior
variance and is otherwise unconstrained by $D_t$. CPTE-EIG is the only one of the three
that scores $c$ as uninformative: $c$ resolves only the direction $\theta_2$, exactly the
direction the target pool's own trade-off cancels.
\end{example}

\section{Query selection algorithm}
\label{app:algorithm}

  \begin{algorithm}[h!]
  \caption{CPTE-EIG query selection at round $t$}
  \label{alg:cpte-eig}
  \begin{algorithmic}[1]
  \renewcommand{\algorithmicrequire}{\textbf{Input:}}
  \renewcommand{\algorithmicensure}{\textbf{Output:}}
  \REQUIRE Potential-outcome laws $\hat{\mu}_0, \hat{\mu}_1$; target
  distribution $p^\star$; quadrature order $n$; Monte Carlo sizes $M$
  (target draws) and $K$ (outcome pairs per target)
  \ENSURE Next query $c_{t+1}$
  \medskip
  \STATE \textbf{Given:} comparisons $\mathcal{D}_t$, candidate pool
  $\mathcal{C}_t$
  \STATE Fit the Laplace-approximate posterior $q_t = \mathrm{GP}(m_t, k_t)$
  from $\mathcal{D}_t$, with the kernel hyperparameters and the noise
  scale $\hat\lambda_t$ maximising the Laplace marginal likelihood (kernel
  amplitude fixed to $1$)
  \STATE Draw target covariates $x_1^\star, \dots, x_M^\star \sim p^\star$,
  and for each $m$, outcome pairs $z_1^{(m)}, \dots, z_K^{(m)} \sim
  \hat{\mu}_1(\cdot \mid x_m^\star) \otimes \hat{\mu}_0(\cdot \mid
  x_m^\star)$ approximating $\nu_{x_m^\star}$
  \FOR{$m = 1, \dots, M$}
      \STATE Compute sensitivity weights $\alpha_t(x_m^\star,
      z_k^{(m)})$, $k = 1, \dots, K$, via \Cref{eq:sensitivity}, and
      estimate $\sigma_{x_m^\star}^2$ by Monte Carlo over $z_1^{(m)},
      \dots, z_K^{(m)}$, via \Cref{eq:sigma-x}
  \ENDFOR
  \FOR{each candidate query $c \in \mathcal{C}_t$}
      \STATE Compute $m_t^\Delta(c)$, $\sigma_c^2 = k_t^\Delta(c, c)$
      \FOR{$m = 1, \dots, M$}
          \STATE Estimate $\sigma_{x_m^\star, c}$ by Monte Carlo over
          $z_1^{(m)}, \dots, z_K^{(m)}$, via \Cref{eq:sigma-c-xc}
          \STATE Compute $\tau_{x_m^\star, c}$, $s_{x_m^\star, c}$ and,
          using the $n$-point Gauss--Hermite rule, $\mathrm{EIG}_{
          \Psi(x_m^\star)}(c)$ via \Cref{eq:quadrature}
      \ENDFOR
      \STATE $\mathrm{CPTE\text{-}EIG}(c) \gets \frac{1}{M} \sum_{m=1}^M
      \mathrm{EIG}_{\Psi(x_m^\star)}(c)$
  \ENDFOR
  \STATE \textbf{return} $c_{t+1} \in \arg\max_{c \in \mathcal{C}_t}
  \mathrm{CPTE\text{-}EIG}(c)$
  \end{algorithmic}
  \end{algorithm}

\section{Additional experimental details and results}
\label{app:experiments}

\subsection{Data generation}

\textbf{Two datasets.} Both datasets are built from PPMI
\citep{marek_Parkinsons_2018}, but they serve different needs. Eliciting
preferences requires outcome profiles a clinician could plausibly be
shown, so the queries use real PPMI profiles: we take every pair of
distinct observed profiles that share the same query covariate $\tilde x$,
which gives 1,369,284 candidate queries. Of these, 2,000 are held out as
a test set and never queried.

Computing an oracle CPTE, on the other hand, requires both potential
outcomes of the same patient, which no real cohort provides. We therefore
simulate a cohort of 2,000 patients from PPMI. Covariates are drawn from
a multivariate Gaussian fitted to the real baseline covariates and
clipped to their observed range. Treatment is assigned by a logistic
propensity model fitted on the real cohort, so overlap is realistic.
$Y^{(h)}(1)$ and $Y^{(h)}(0)$ are obtained by nearest-neighbour matching
into the real treated (dopamine agonist) and control (levodopa) arms.
No method ever sees these ground-truth quantities.

\textbf{Outcome laws.} For a target patient $x$, let $N_k(x)$ be its
$k=30$ nearest cohort patients under a Gower distance on the full
baseline covariates, excluding $x$ itself. We set
$\hat\mu_a(\cdot\mid x)$ to the empirical distribution of
$\{Y^{(h)}_j(a)\}_{j\in N_k(x)}$. The law $\nu^{(h)}_x=\hat\mu_1\otimes\hat\mu_0$
of \Cref{eq:nu} is then uniform over the $k^2=900$ (treated, control)
pairs, so every integral against $\nu^{(h)}_x$ used for evaluation is
computed exactly by enumeration. The acquisition rules only sample from
$\nu^{(h)}_x$.

\textbf{Follow-up horizons.} Treatment effects in Parkinson's disease are
reported at horizons fixed in advance with clinicians, here
$h\in\{1,3,5,7\}$ years. For each $h$, $Y^{(h)}(0)$ and $Y^{(h)}(1)$ are
the profiles recorded at that horizon, and a separate estimand is
defined from the same covariates. The target distribution $p^\star$ is
uniform over (patient, horizon) pairs, so a single run serves all four
horizons.

\textbf{Treatment effects.} We use three versions of the cohort, which
differ only in $Y(1)$.
\begin{itemize}
\item \emph{Tiny effect:} the matched outcomes above. The two arms have
similar outcome laws. Over the 100 evaluation patients, the mean oracle
CPTE is 0.49--0.56 depending on the horizon.
\item \emph{Mild} ($s=1$) and \emph{strong} ($s=2$) effects: $Y(0)$ is
kept and
\[
Y^{(h)}(1)=\mathrm{clip}\big(Y^{(h)}(0)+s\,\delta(\tilde x,h),\,0,\,4\big),
\]
where $\delta$ is zero on every item except two. Tremor (NP2TRMR)
improves by one point, $\delta_{\text{tremor}}=-1$. Cognition (NP1COG)
worsens by $\delta_{\text{cog}}=(0.5+0.15\,h)(0.5+a)$, with
$a=(\text{age group}-1)/7\in[0,1]$ the normalised age. The harm
therefore grows with the time on treatment and with age.
\end{itemize}
The effect is deterministic given $(\tilde x,h)$ and additive on the
0--4 item scale. Clipping creates a floor: a patient without tremor
cannot improve on it. Both items vary a lot across the query pool, so
the utility weights they depend on can be learned from ordinary
queries. The magnitudes were chosen from properties of the oracle CPTE
only, so that at every horizon some patients benefit and some do not,
never from the performance of any method.

\subsection{The simulated clinician}

Each query is answered as the model of Section~\ref{sec:model} assumes,
$V_c\sim\mathrm{Bernoulli}\big(\Phi(D_c u^\star/\lambda^\star)\big)$, with
$\lambda^\star=0.1\sqrt 2$. The utility is
\[
u^\star(\tilde x,y)=-\sum_{d} w_d(\tilde x)\sum_{j\in G_d}y_j,
\]
with five domains $G_d$. The weights are positive and were drawn once
from fixed ranges:
\begin{itemize}
\item neuropsychiatric (NP1COG, NP1HALL, NP1DPRS, NP1APAT, NP1DDS):
$w=0.38+1.06\,a$, increasing with age;
\item motor (NP2TRMR, NP2WALK, NP2FREZ): $w=0.36+1.20\,(1-a)$,
decreasing with age;
\item motor complications (NP4DYSKI, NP4FLCTI): $w=(0.40,0.64)+1.15\,
\text{sex}$, with sex coded $0/1$;
\item sleep and autonomic (NP1SLPD, NP1LTHD): $w=(0.49,0.38)$;
\item quality of life: $w=1.22$.
\end{itemize}
Since all weights are positive, $u^\star$ is decreasing in every symptom.
Learning $u^\star$ everywhere needs comparisons over all ages and both
sexes, while $\Psi_x$ only depends on the weights at $\tilde x$, because
both profiles of a treated--control pair share the same $\tilde x$.

\subsection{Methods and protocol}

\textbf{Shared model.} All methods use the same model. It is a PairwiseGP from \emph{botorch} library with Laplace approximation \citep{chu_Gaussian_2005}, on inputs rescaled
to $[0,1]$. The kernel is a mixed kernel
$k=\theta\,(k_{\mathrm{RBF}}+k_{\mathrm{cat}})+(1-\theta)\,k_{\mathrm{RBF}}\,k_{\mathrm{cat}}$
\citep{ru2020bayesian}. $k_{\mathrm{RBF}}$ is an ARD RBF kernel on age
group, disease duration and the 13 outcomes, with lengthscales at least
$0.05$, and $k_{\mathrm{cat}}$ is a categorical kernel on sex.
As in \citet{chu_Gaussian_2005}, the kernel amplitude is fixed, which is
needed for $\lambda$ to be identifiable.

\textbf{Noise scale.} $\lambda$ and the kernel amplitude are not separately
identifiable. We therefore fix the output scale to 1 and treat $\lambda\in[10^{-3},10^{3}]$ as a hyperparameter. At every round,
the lengthscales, $\theta$ and $\lambda$ are fitted jointly by maximising
the Laplace marginal likelihood with L-BFGS, starting from $\lambda=1$.
The estimate $\hat\lambda_t$ is then used everywhere: in the likelihood,
in each acquisition score, and in the CPTE estimate. The oracle alone
uses $\lambda^\star=0.1\sqrt{2}$. 

\textbf{Failed fits.} If the optimiser fails, we refit with $\lambda$ held
at the previous round's estimate, and if that also fails, the round's
query is drawn at random from the pool. The same rule applies to every
method and to the final evaluation fit. Across all runs, fewer than 1\%
of fits needed the held-$\lambda$ refit, and a random query was used 10
times in total, spread across methods.

\textbf{Acquisition rules.} We write $V_c$ for the answer to a query
$c$. All methods start from the same 20 random comparisons, so they
coincide at the smallest budget. At each later round, each method scores
2,000 unlabelled comparisons drawn at random from the pool and queries
the best one.
\begin{itemize}
\item \emph{Random} queries a comparison drawn uniformly from the pool.
\item \emph{BALD} \citep{houlsby_Bayesian_2011b} scores the information
the answer brings about the utility, $I(V_c;u\mid\mathcal D_t)$.
We use BoTorch's BALD for pairwise GPs (\texttt{PairwiseBayesianActiveLearningByDisagreement}), with
  the learned $\hat\lambda_t$ in place of its fixed $\lambda=1$, and its Monte Carlo expectation replaced by the
  deterministic 40-node Gauss–Hermite rule used for CPTE-EIG, which removes sampling noise.
\item \emph{EPIG} \citep{smith_PredictionOriented_2023b} scores the
information the answer brings about the answers to target queries. We
give it the treated--control queries behind the CPTE. For a target
$x\sim p^\star$ and a pair $z=(z_1,z_0)\sim\nu_x$, let $c^\star(x,z)$
compare the treated profile $(\tilde x,z_1)$ with the control profile
$(\tilde x,z_0)$. Its answer $V_{c^\star(x,z)}$ says whether the clinician
prefers the treated outcome, and $\Psi_x$ is the average of its
probability over $z\sim\nu_x$. Then
\[
\mathrm{EPIG}(c)=\mathbb E_{x\sim p^\star,\,z\sim\nu_x}
\big[I(V_c;V_{c^\star(x,z)}\mid\mathcal D_t)\big], \
\mathrm{CPTE\text{-}EIG}(c)=\mathbb E_{x\sim p^\star}
\big[I(V_c;\Psi_x\mid\mathcal D_t)\big].
\]
Both criteria use the same targets, drawn from the same law. They differ
only in what they aim to learn about them: each individual answer for
EPIG, and their average $\Psi_x$ for CPTE-EIG. EPIG uses 200 target queries per round, each
from a fresh draw of $x$ and $z$, and estimates each mutual information
with 512 Monte Carlo samples from the joint GP posterior at $c$ and
$c^\star$.
\item \emph{CPTE-EIG} follows Algorithm~\ref{alg:cpte-eig}, with $M=20$
targets $x\sim p^\star$ and $K=20$ pairs $z\sim\nu_x$ per target, drawn at every round, and an $n=40$-node Gauss--Hermite rule.
\end{itemize}

\textbf{Budget and repeats.} Budgets range from 20 to 150 comparisons,
with 15 seeds per budget. After $n$ comparisons, the GP is refitted on
all of them and evaluated.

\subsection{Metrics}

All CPTE metrics use the same 100 evaluation patients, drawn once from
the cohort and shared by all seeds and methods. The estimate is
$T^{(h)}_x(m_t)=\frac{1}{k^2}\sum_{i,j\in N_k(x)}
\Phi\big(m^\Delta_t(x,z_{ij})/\hat\lambda_t\big)$, and the oracle
$\Psi^{(h)}_x$ is the same average with $u^\star$ and $\lambda^\star$,
both over the 900 pairs of $\nu^{(h)}_x$.
\begin{itemize}
\item \emph{CPTE MSE:} the mean of $(T^{(h)}_x-\Psi^{(h)}_x)^2$ over the
100 patients, per horizon.
\item \emph{Policy regret:} $V(\pi^\star)-V(\pi_t)$, where
$V(\pi)=\frac{1}{100}\sum_x\big[\pi(x)\Psi_x+(1-\pi(x))(1-\Psi_x)\big]$
and $\pi_t(x)=\mathbf 1\{T_x(m_t)>1/2\}$. Each wrong decision is
weighted by $|2\Psi_x-1|$, so errors on clear-cut patients cost more.
\item \emph{Held-out prediction:} accuracy, Brier score and negative
log-likelihood of the predicted answers on the 2,000 test comparisons,
against answers simulated from the same clinician.
\end{itemize}
Tables and bands report the mean over the 15 seeds with a 95\%
percentile-bootstrap confidence interval for the mean (3,000 resamples).

\subsection{Tiny and strong treatment effects}

\begin{figure}[t]
\centering
\includegraphics[width=\textwidth]{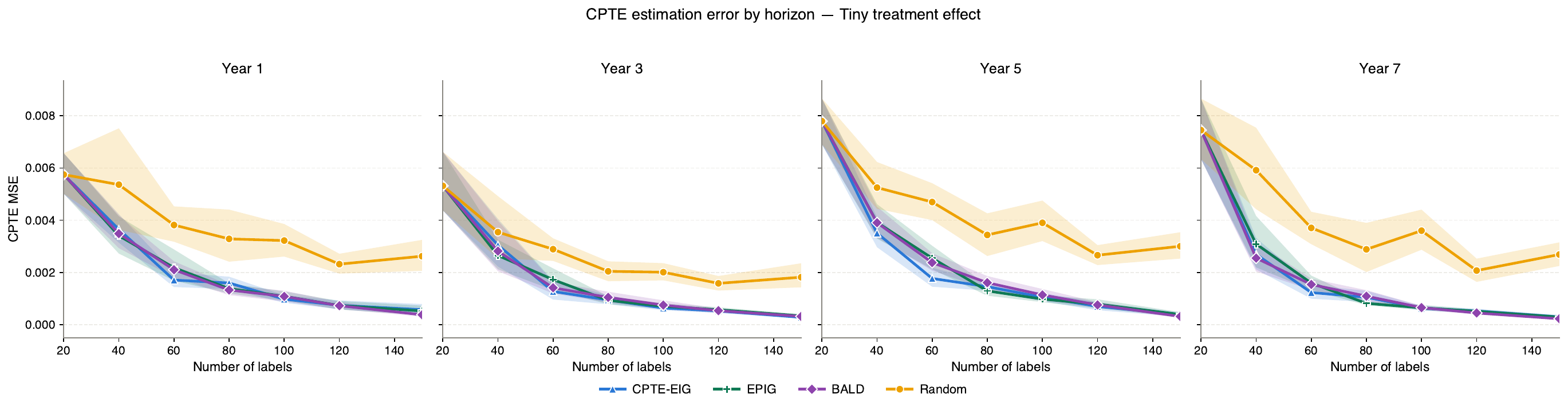}\\
\includegraphics[width=\textwidth]{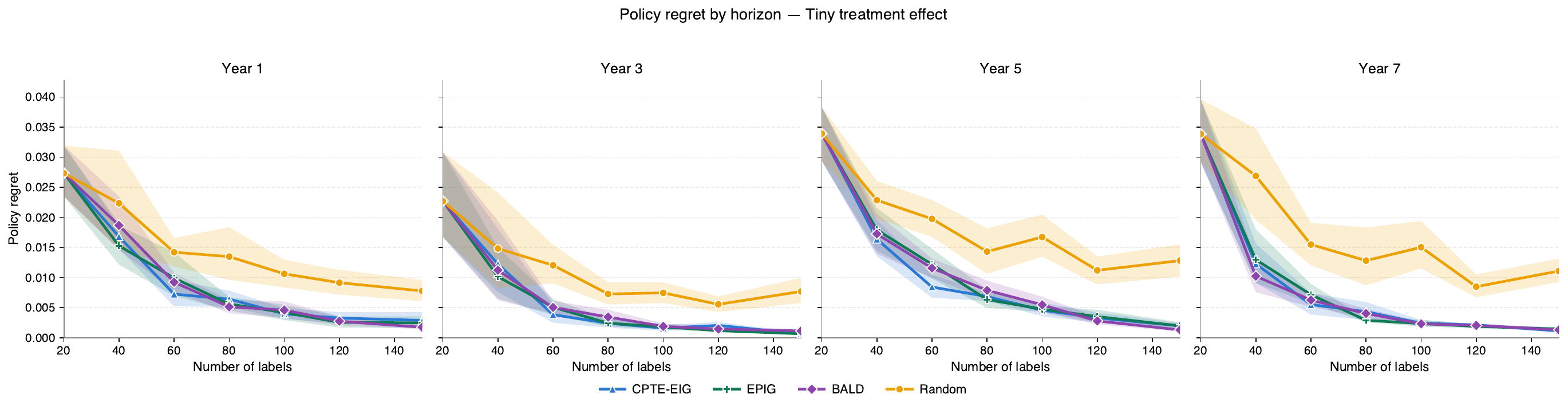}
\caption{Tiny treatment effect. CPTE MSE (top) and policy regret
(bottom) per horizon, as in Figure~\ref{fig:effect-cpte}.}
\label{fig:tiny-cpte}
\end{figure}

\begin{figure}[t]
\centering
\includegraphics[width=\textwidth]{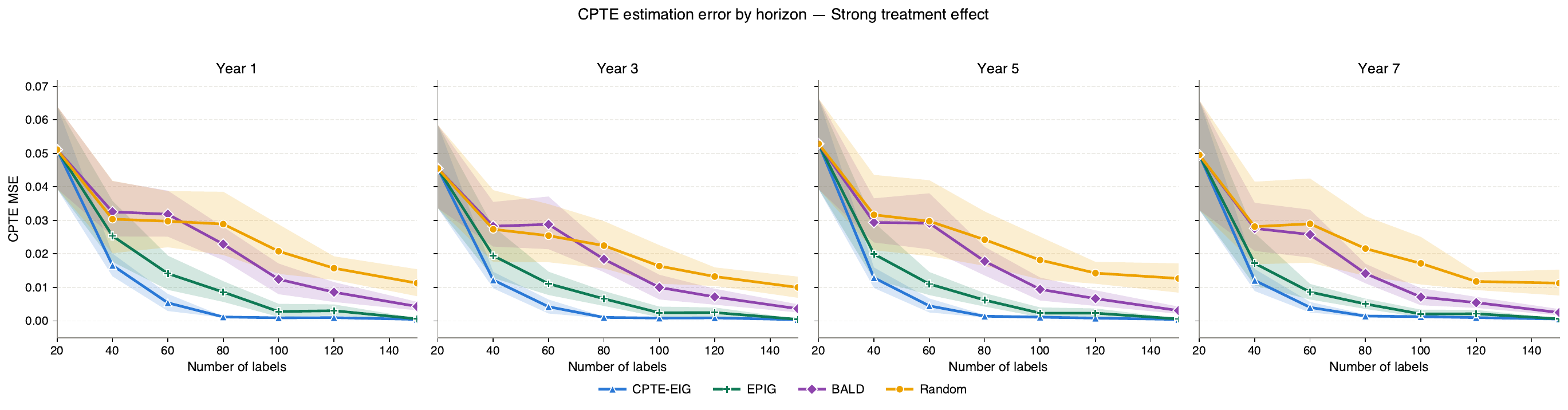}\\
\includegraphics[width=\textwidth]{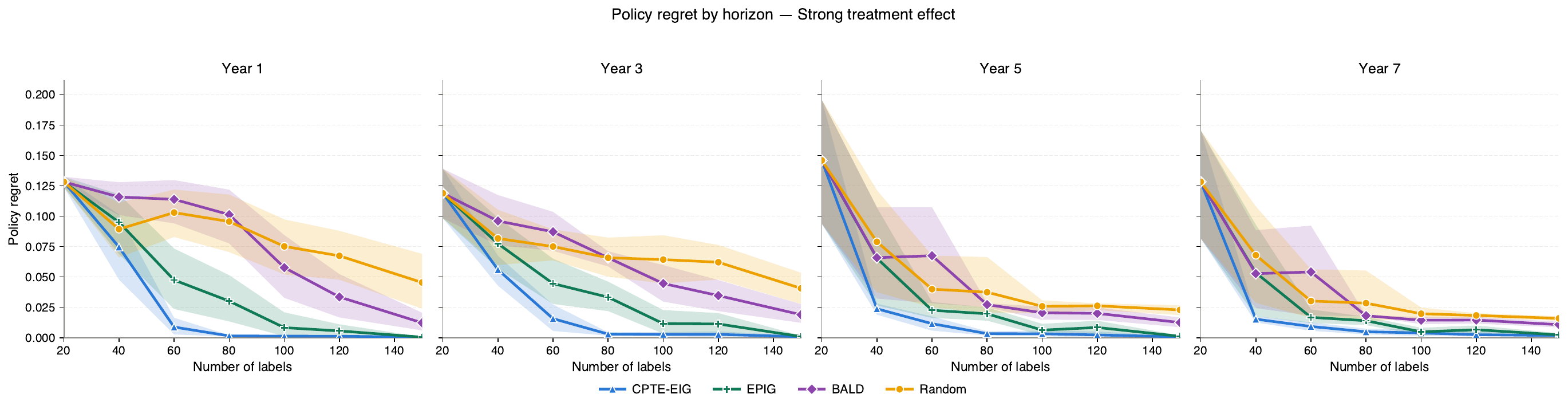}
\caption{Strong treatment effect. CPTE MSE (top) and policy regret
(bottom) per horizon, as in Figure~\ref{fig:effect-cpte}.}
\label{fig:strong-cpte}
\end{figure}

With a strong effect (Figure~\ref{fig:strong-cpte}), the ordering of the
mild case holds at every horizon and every budget from 40 comparisons
on, for both metrics, with wider gaps. At 80 comparisons, CPTE-EIG's
regret is $0.003$, against $0.024$ for EPIG and $0.053$ for BALD. On
CPTE MSE, it beats BALD on at least 14 of the 15 seeds at every budget
from 40 on, and EPIG on 12 of 15 seeds at 80.

With a tiny effect (Figure~\ref{fig:tiny-cpte}), CPTE-EIG, EPIG and
BALD overlap at every horizon. The estimand is easy here: after the 20
initial random comparisons, the CPTE MSE is already $6.6\times10^{-3}$,
against $2.1\times10^{-2}$ with a mild effect and $5.0\times10^{-2}$ with
a strong one. Little is left for any acquisition rule to gain.

\subsection{Learning $w$ is not learning $\Psi_x$}

\begin{figure}[t]
\centering
\includegraphics[width=\textwidth]{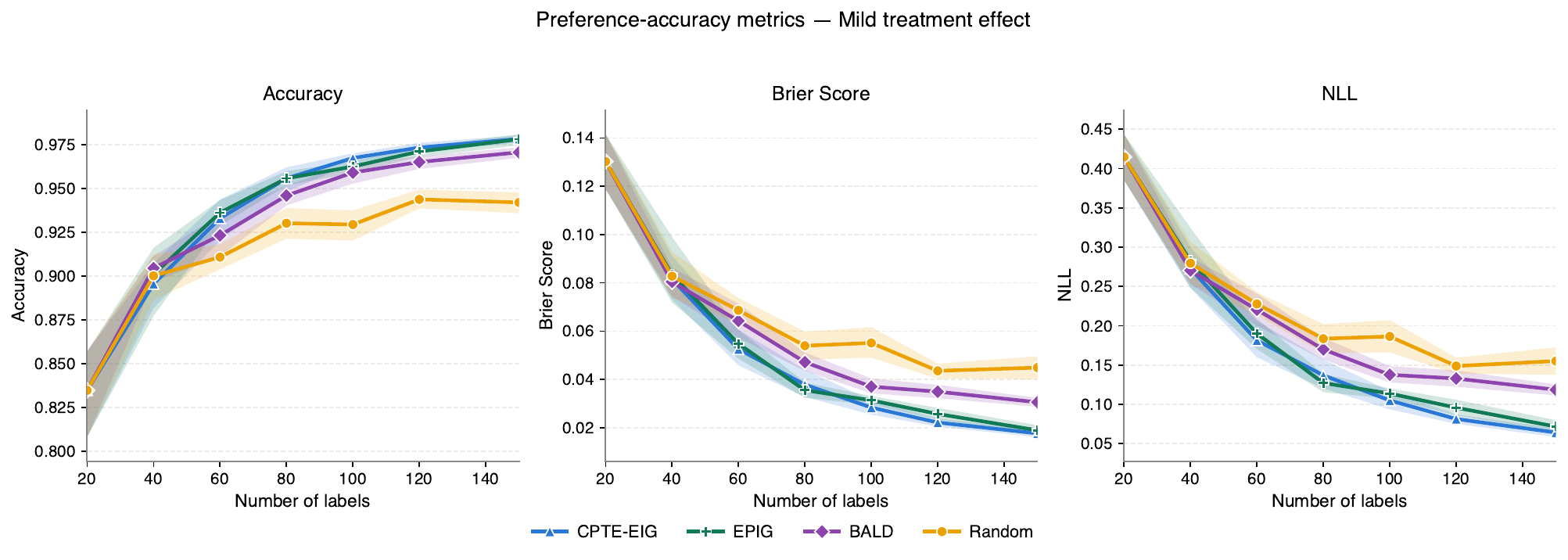}
\caption{Mild treatment effect. Accuracy (higher is better), Brier score
and negative log-likelihood (lower is better) of the predicted answers on
the 2,000 held-out comparisons. These measure the fitted preference rule,
not the estimand.}
\label{fig:pref-metrics}
\end{figure}

Figure~\ref{fig:pref-metrics} reports held-out prediction under the mild
effect. CPTE-EIG and EPIG are within 0.01 accuracy of each other at every
budget, and BALD is at most 0.01 behind them from 80 comparisons on. On
CPTE MSE, the same methods differ by factors of 2 to 5 between 40 and
120 comparisons (Table~\ref{tab:summary}).

BALD and Random never use the outcomes of
the target patients, so their queries and fitted models are identical
under the three treatment effects. We checked that their held-out
metrics coincide exactly across the three runs. Only their CPTE error
changes, because the estimand does. The same fitted utility can thus be
good for one causal question and poor for another. What matters is how
well it is known along the contrasts $\Psi_x$ depends on.

\subsection{Gauss--Hermite quadrature versus Monte Carlo}
\label{app:quadrature-comparison}


Equation \ref{eq:quadrature} approximates the remaining expectation over
$\Psi_x\sim\Pi_{t,x}$ by an $n$-point Gauss--Hermite rule rather than by
direct Monte Carlo sampling of $\Psi_x$. We check that this choice does
not trade accuracy for speed. We rerun CPTE-EIG with four rules for this
one expectation, all else unchanged: Gauss--Hermite with $n=10$ and
$n=40$ nodes, and plain Monte Carlo with 2,000 and 8,000 samples from
$\mathcal N(T_x(m_t),\sigma_x^2)$.

\begin{figure*}[t]
\centering
\includegraphics[width=\textwidth]{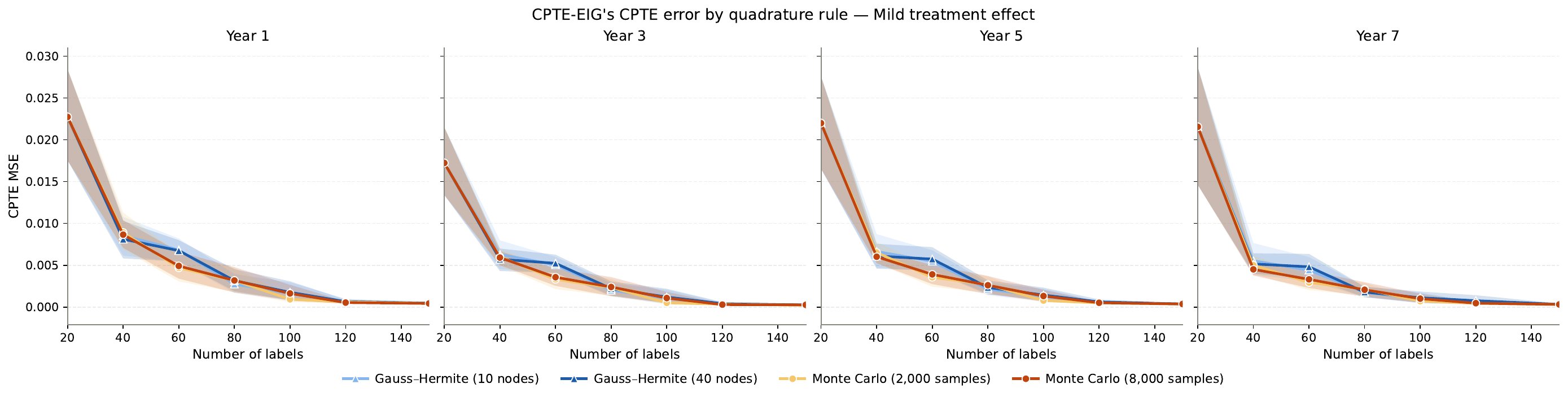}
\caption{CPTE MSE against the oracle, per horizon, for four rules
approximating the expectation in
\Cref{eq:eig-answer-side}--\Cref{eq:quadrature}: Gauss--Hermite
quadrature with $n=10$ and $n=40$ nodes, and Monte Carlo with 2,000 and
8,000 samples of $\Psi_x$.}
\label{fig:quadrature-cpte-mse}
\end{figure*}

\begin{figure}[t]
\centering
\includegraphics[width=0.8\columnwidth]{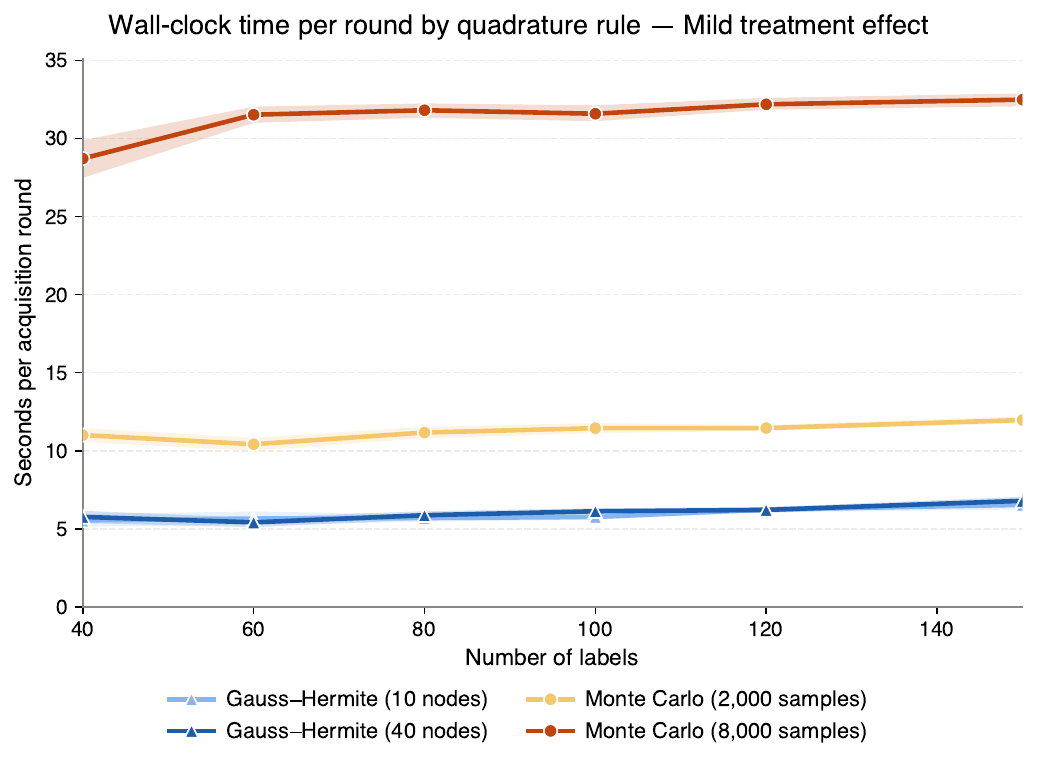}
\caption{Median wall-clock time per acquisition round (log scale), for
the same four rules as Figure~\ref{fig:quadrature-cpte-mse}.}
\label{fig:quadrature-wallclock}
\end{figure}

Figure~\ref{fig:quadrature-cpte-mse} shows the four rules give
statistically indistinguishable CPTE MSE at every budget and horizon:
even $n=10$ nodes match 8,000 Monte Carlo samples. This is expected.
Under the linearisation of Theorem~\ref{thm:linearisation}, $\Pi_{t,x}$
is Gaussian, and Gauss--Hermite quadrature is the natural rule for
integrating against a Gaussian density, so few nodes already suffice.
The two rules differ in cost, however (Figure~\ref{fig:quadrature-wallclock}).
Gauss--Hermite is several times cheaper than Monte Carlo at every budget. We use $n=40$ nodes in all
other experiments: the same accuracy as Monte Carlo, at a fraction of its cost.

\stopcontents[appendices]
\end{appendices}

\end{document}

%% file: math_commands.tex
\usepackage{amsmath,amsfonts,bm}

\def\eqref#1{equation~\ref{#1}}

\def\1{\bm{1}}

\DeclareMathAlphabet{\mathsfit}{\encodingdefault}{\sfdefault}{m}{sl}
\SetMathAlphabet{\mathsfit}{bold}{\encodingdefault}{\sfdefault}{bx}{n}



%% file: iclr2027_conference.bib
@article{butler2018incorporating,
  title={Incorporating patient preferences into estimation of optimal individualized treatment rules},
  author={Butler, Emily L and Laber, Eric B and Davis, Sonia M and Kosorok, Michael R},
  journal={Biometrics},
  volume={74},
  number={1},
  pages={18--26},
  year={2018},
  publisher={Oxford University Press}
}

@article{nie2021quasi,
  title={Quasi-oracle estimation of heterogeneous treatment effects},
  author={Nie, Xinkun and Wager, Stefan},
  journal={Biometrika},
  volume={108},
  number={2},
  pages={299--319},
  year={2021},
  publisher={Oxford University Press}
}

@inproceedings{bengs_Stochastic_2022,
  title = {Stochastic {{Contextual Dueling Bandits}} under {{Linear Stochastic Transitivity Models}}},
  booktitle = {Proceedings of the 39th {{International Conference}} on {{Machine Learning}}},
  author = {Bengs, Viktor and Saha, Aadirupa and H{\"u}llermeier, Eyke},
  year = 2022,
  month = jun,
  pages = {1764--1786},
  publisher = {PMLR},
  issn = {2640-3498},
  urldate = {2026-09-18},
  langid = {english}
}

@article{bradley_Rank_1952,
  title = {Rank {{Analysis}} of {{Incomplete Block Designs}}: {{I}}. {{The Method}} of {{Paired Comparisons}}},
  shorttitle = {Rank {{Analysis}} of {{Incomplete Block Designs}}},
  author = {Bradley, Ralph Allan and Terry, Milton E.},
  year = 1952,
  journal = {Biometrika},
  volume = {39},
  number = {3/4},
  eprint = {2334029},
  eprinttype = {jstor},
  pages = {324--345},
  publisher = {[Oxford University Press, Biometrika Trust]},
  issn = {0006-3444},
  doi = {10.2307/2334029},
  urldate = {2026-09-18}
}

@article{buyse_Generalized_2010,
  title = {Generalized Pairwise Comparisons of Prioritized Outcomes in the Two-Sample Problem},
  author = {Buyse, Marc},
  year = 2010,
  month = dec,
  journal = {Statistics in Medicine},
  volume = {29},
  number = {30},
  pages = {3245--3257},
  issn = {1097-0258},
  doi = {10.1002/sim.3923},
  langid = {english},
  pmid = {21170918}
}

@inproceedings{chau_Learning_2022,
  title = {Learning {{Inconsistent Preferences}} with {{Gaussian Processes}}},
  booktitle = {Proceedings of {{The}} 25th {{International Conference}} on {{Artificial Intelligence}} and {{Statistics}}},
  author = {Chau, Siu Lun and Gonzalez, Javier and Sejdinovic, Dino},
  year = 2022,
  month = may,
  pages = {2266--2281},
  publisher = {PMLR},
  issn = {2640-3498},
  urldate = {2026-09-18},
  langid = {english}
}

@inproceedings{chu2005preference,
  title={Preference learning with Gaussian processes},
  author={Chu, Wei and Ghahramani, Zoubin},
  booktitle={Proceedings of the 22nd international conference on Machine learning},
  pages={137--144},
  year={2005}
}

@article{chu_Gaussian_2005,
  title = {Gaussian {{Processes}} for {{Ordinal Regression}}},
  author = {Chu, Wei and Ghahramani, Zoubin},
  year = 2005,
  journal = {Journal of Machine Learning Research},
  volume = {6},
  number = {35},
  pages = {1019--1041},
  issn = {1533-7928},
  urldate = {2026-09-18}
}

@book{cover_Elements_1991,
  title = {Elements of Information Theory},
  author = {Cover, T. M. and Thomas, Joy A.},
  year = 1991,
  series = {Wiley Series in Telecommunications},
  publisher = {Wiley},
  address = {New York},
  isbn = {978-0-471-06259-2},
  langid = {english},
  lccn = {Q360 .C68 1991}
}

@inproceedings{eric_Active_2007,
  title = {Active {{Preference Learning}} with {{Discrete Choice Data}}},
  booktitle = {Advances in {{Neural Information Processing Systems}}},
  author = {Brochu, Eric and Freitas, Nando and Ghosh, Abhijeet},
  year = 2007,
  volume = {20},
  publisher = {Curran Associates, Inc.},
  urldate = {2026-09-18}
}

@article{even_Rethinking_2026,
  title={Rethinking the win ratio: A causal framework for hierarchical outcome analysis},
  author={Even, Mathieu and Josse, Julie},
  journal={arXiv preprint arXiv:2501.16933},
  year={2025}
}

@book{folland_Real_2013,
  title = {Real {{Analysis}}: {{Modern Techniques}} and {{Their Applications}}},
  shorttitle = {Real {{Analysis}}},
  author = {Folland, Gerald B.},
  year = 2013,
  month = jun,
  publisher = {John Wiley \& Sons},
  googlebooks = {wI4fAwAAQBAJ},
  isbn = {978-1-118-62639-9},
  langid = {english}
}

@article{freemantle_Composite_2003,
  title = {Composite Outcomes in Randomized Trials: Greater Precision but with Greater Uncertainty?},
  shorttitle = {Composite Outcomes in Randomized Trials},
  author = {Freemantle, Nick and Calvert, Melanie and Wood, John and Eastaugh, Joanne and Griffin, Carl},
  year = 2003,
  month = may,
  journal = {JAMA},
  volume = {289},
  number = {19},
  pages = {2554--2559},
  issn = {0098-7484},
  doi = {10.1001/jama.289.19.2554},
  langid = {english},
  pmid = {12759327}
}

@inproceedings{gao_CausalEPIG_2025a,
  title={Causal-EPIG: Causally Aligned Active CATE Estimation},
  author={Gao, Erdun and Fawkes, Jake and Sejdinovic, Dino},
  booktitle={Forty-third International Conference on Machine Learning},
  year=2026
}

@article{goetz_Movement_2008,
  title = {Movement {{Disorder Society-sponsored}} Revision of the {{Unified Parkinson}}'s {{Disease Rating Scale}} ({{MDS-UPDRS}}): Scale Presentation and Clinimetric Testing Results},
  shorttitle = {Movement {{Disorder Society-sponsored}} Revision of the {{Unified Parkinson}}'s {{Disease Rating Scale}} ({{MDS-UPDRS}})},
  author = {Goetz, Christopher G. and Tilley, Barbara C. and Shaftman, Stephanie R. and Stebbins, Glenn T. and Fahn, Stanley and {Martinez-Martin}, Pablo and Poewe, Werner and Sampaio, Cristina and Stern, Matthew B. and Dodel, Richard and Dubois, Bruno and Holloway, Robert and Jankovic, Joseph and Kulisevsky, Jaime and Lang, Anthony E. and Lees, Andrew and Leurgans, Sue and LeWitt, Peter A. and Nyenhuis, David and Olanow, C. Warren and Rascol, Olivier and Schrag, Anette and Teresi, Jeanne A. and {van Hilten}, Jacobus J. and LaPelle, Nancy and {Movement Disorder Society UPDRS Revision Task Force}},
  year = 2008,
  month = nov,
  journal = {Movement Disorders: Official Journal of the Movement Disorder Society},
  volume = {23},
  number = {15},
  pages = {2129--2170},
  issn = {1531-8257},
  doi = {10.1002/mds.22340},
  langid = {english},
  pmid = {19025984}
}

@article{gongora-salazar_Use_2023,
  title = {The {{Use}} of {{Multicriteria Decision Analysis}} to {{Support Decision Making}} in {{Healthcare}}: {{An Updated Systematic Literature Review}}},
  shorttitle = {The {{Use}} of {{Multicriteria Decision Analysis}} to {{Support Decision Making}} in {{Healthcare}}},
  author = {{Gongora-Salazar}, Pamela and Rocks, Stephen and Fahr, Patrick and {Rivero-Arias}, Oliver and Tsiachristas, Apostolos},
  year = 2023,
  month = may,
  journal = {Value in Health},
  volume = {26},
  number = {5},
  pages = {780--790},
  publisher = {Elsevier},
  issn = {1098-3015, 1524-4733},
  doi = {10.1016/j.jval.2022.11.007},
  urldate = {2026-09-18},
  langid = {english},
  pmid = {36436791}
}

@book{hernan_Causal_2025,
  title = {Causal {{Inference}}: {{What If}}},
  shorttitle = {Causal {{Inference}}},
  author = {Hernan, Miguel A. and Robins, James M.},
  year = 2025,
  month = jul,
  publisher = {CRC Press},
  googlebooks = {\_KnHIAAACAAJ},
  isbn = {978-1-4200-7616-5},
  langid = {english}
}

@article{houlsby_Bayesian_2011b,
  title={Bayesian active learning for classification and preference learning},
  author={Houlsby, Neil and Husz{\'a}r, Ferenc and Ghahramani, Zoubin and Lengyel, M{\'a}t{\'e}},
  journal={arXiv preprint arXiv:1112.5745},
  year={2011}
}

@inproceedings{houlsby_Collaborative_2012,
  title = {Collaborative {{Gaussian Processes}} for {{Preference Learning}}},
  booktitle = {Advances in {{Neural Information Processing Systems}}},
  author = {Houlsby, Neil and Huszar, Ferenc and Ghahramani, Zoubin and {Hern{\'a}ndez-lobato}, Jose},
  year = 2012,
  volume = {25},
  publisher = {Curran Associates, Inc.},
  urldate = {2026-09-18}
}

@article{kalia_Parkinsons_2015,
  title = {Parkinson's Disease},
  author = {Kalia, Lorraine V. and Lang, Anthony E.},
  year = 2015,
  month = aug,
  journal = {Lancet},
  volume = {386},
  number = {9996},
  pages = {896--912},
  address = {London, England},
  issn = {1474-547X},
  doi = {10.1016/S0140-6736(14)61393-3},
  langid = {english},
  pmid = {25904081}
}

@article{liu_Note_1994,
  title = {A {{Note}} on {{Gauss-Hermite Quadrature}}},
  author = {Liu, Qing and Pierce, Donald A.},
  year = 1994,
  journal = {Biometrika},
  volume = {81},
  number = {3},
  eprint = {2337136},
  eprinttype = {jstor},
  pages = {624--629},
  publisher = {[Oxford University Press, Biometrika Trust]},
  issn = {0006-3444},
  doi = {10.2307/2337136},
  urldate = {2026-09-18}
}

@article{marek_Parkinsons_2018,
  title = {The {{Parkinson}}'s Progression Markers Initiative ({{PPMI}}) -- Establishing a {{PD}} Biomarker Cohort},
  author = {Marek, Kenneth and Chowdhury, Sohini and Siderowf, Andrew and Lasch, Shirley and Coffey, Christopher S. and Caspell-Garcia, Chelsea and Simuni, Tanya and Jennings, Danna and Tanner, Caroline M. and Trojanowski, John Q. and Shaw, Leslie M. and Seibyl, John and Schuff, Norbert and Singleton, Andrew and Kieburtz, Karl and Toga, Arthur W. and Mollenhauer, Brit and Galasko, Doug and Chahine, Lana M. and Weintraub, Daniel and Foroud, Tatiana and Tosun-Turgut, Duygu and Poston, Kathleen and Arnedo, Vanessa and Frasier, Mark and Sherer, Todd},
  year = 2018,
  month = oct,
  journal = {Annals of Clinical and Translational Neurology},
  volume = {5},
  number = {12},
  pages = {1460--1477},
  issn = {2328-9503},
  doi = {10.1002/acn3.644},
  urldate = {2026-09-18},
  pmcid = {PMC6292383},
  pmid = {30564614}
}

@inproceedings{parnas_Preferencebased_2026b,
  title = {Preference-Based {{Conditional Treatment Effects}} and {{Policy Learning}}},
  booktitle = {Proceedings of {{The}} 29th {{International Conference}} on {{Artificial Intelligence}} and {{Statistics}}},
  author = {Parnas, Dovid and Even, Mathieu and Josse, Julie and Shalit, Uri},
  year = 2026,
  month = aug,
  pages = {226--234},
  publisher = {PMLR},
  issn = {2640-3498},
  urldate = {2026-09-18},
  langid = {english}
}

@article{pocock_win_2012a,
  title = {The Win Ratio: A New Approach to the Analysis of Composite Endpoints in Clinical Trials Based on Clinical Priorities},
  shorttitle = {The Win Ratio},
  author = {Pocock, Stuart J. and Ariti, Cono A. and Collier, Timothy J. and Wang, Duolao},
  year = 2012,
  month = jan,
  journal = {European Heart Journal},
  volume = {33},
  number = {2},
  pages = {176--182},
  issn = {1522-9645},
  doi = {10.1093/eurheartj/ehr352},
  langid = {english},
  pmid = {21900289}
}

@book{rasmussen_Gaussian_2005a,
  title = {Gaussian {{Processes}} for {{Machine Learning}}},
  author = {Rasmussen, Carl Edward and Williams, Christopher K. I.},
  year = 2005,
  month = nov,
  publisher = {The MIT Press},
  doi = {10.7551/mitpress/3206.001.0001},
  urldate = {2026-09-18},
  isbn = {978-0-262-25683-4},
  langid = {english}
}

@incollection{robert_Monte_2011,
  title = {Monte {{Carlo Methods}} in {{Statistics}}},
  booktitle = {International {{Encyclopedia}} of {{Statistical Science}}},
  author = {Robert, Christian},
  year = 2011,
  pages = {854--858},
  publisher = {Springer, Berlin, Heidelberg},
  doi = {10.1007/978-3-642-04898-2_376},
  urldate = {2026-09-18},
  isbn = {978-3-642-04898-2},
  langid = {english}
}

@article{rosenbaum_Central_1983,
  title = {The {{Central Role}} of the {{Propensity Score}} in {{Observational Studies}} for {{Causal Effects}}},
  author = {Rosenbaum, Paul R. and Rubin, Donald B.},
  year = 1983,
  journal = {Biometrika},
  volume = {70},
  number = {1},
  eprint = {2335942},
  eprinttype = {jstor},
  pages = {41--55},
  publisher = {[Oxford University Press, Biometrika Trust]},
  issn = {0006-3444},
  doi = {10.2307/2335942},
  urldate = {2026-09-18}
}

@article{rubin_Randomization_1980,
  title = {Randomization {{Analysis}} of {{Experimental Data}}: {{The Fisher Randomization Test Comment}}},
  shorttitle = {Randomization {{Analysis}} of {{Experimental Data}}},
  author = {Rubin, Donald B.},
  year = 1980,
  journal = {Journal of the American Statistical Association},
  volume = {75},
  number = {371},
  eprint = {2287653},
  eprinttype = {jstor},
  pages = {591--593},
  publisher = {[American Statistical Association, Taylor \& Francis, Ltd.]},
  issn = {0162-1459},
  doi = {10.2307/2287653},
  urldate = {2026-09-18}
}

@inproceedings{smith_PredictionOriented_2023b,
  title={Prediction-oriented bayesian active learning},
  author={Smith, Freddie Bickford and Kirsch, Andreas and Farquhar, Sebastian and Gal, Yarin and Foster, Adam and Rainforth, Tom},
  booktitle={International conference on artificial intelligence and statistics},
  pages={7331--7348},
  year={2023},
  organization={PMLR}
}

@article{thurstone_law_1927,
  title = {A Law of Comparative Judgment},
  author = {Thurstone, L. L.},
  year = 1927,
  journal = {Psychological Review},
  volume = {34},
  number = {4},
  pages = {273--286},
  publisher = {Psychological Review Company},
  address = {US},
  issn = {1939-1471},
  doi = {10.1037/h0070288}
}

@article{williams_Bayesian_1998,
  title = {Bayesian Classification with {{Gaussian}} Processes},
  author = {Williams, C.K.I. and Barber, D.},
  year = 1998,
  month = dec,
  journal = {IEEE Transactions on Pattern Analysis and Machine Intelligence},
  volume = {20},
  number = {12},
  pages = {1342--1351},
  issn = {01628828},
  doi = {10.1109/34.735807},
  urldate = {2026-09-18},
  copyright = {https://ieeexplore.ieee.org/Xplorehelp/downloads/license-information/IEEE.html},
  langid = {english}
}

@article{petit2025optimal,
  title={Optimal treatment regimes for the net benefit of a treatment},
  author={Petit, Fran{\c{c}}ois and Biau, G{\'e}rard and Porcher, Rapha{\"e}l},
  journal={arXiv preprint arXiv:2503.22580},
  year={2025}
}

@article{rubin1974estimating,
  title={Estimating causal effects of treatments in randomized and nonrandomized studies.},
  author={Rubin, Donald B},
  journal={Journal of educational Psychology},
  volume={66},
  number={5},
  pages={688},
  year={1974},
  publisher={American Psychological Association}
}

@misc{angrist1995identification,
  title={Identification and estimation of local average treatment effects},
  author={Angrist, Joshua and Imbens, Guido},
  year={1995},
  publisher={National Bureau of Economic Research Cambridge, Mass., USA}
}

@article{rainforth2024modern,
  title={Modern Bayesian experimental design},
  author={Rainforth, Tom and Foster, Adam and Ivanova, Desi R and Bickford Smith, Freddie},
  journal={Statistical Science},
  volume={39},
  number={1},
  pages={100--114},
  year={2024},
  publisher={Institute of Mathematical Statistics}
}

@article{lindley1956measure,
  title={On a measure of the information provided by an experiment},
  author={Lindley, Dennis V},
  journal={The Annals of Mathematical Statistics},
  volume={27},
  number={4},
  pages={986--1005},
  year={1956},
  publisher={Institute of Mathematical Statistics}
}

@article{chaloner1995bayesian,
  title={Bayesian experimental design: A review},
  author={Chaloner, Kathryn and Verdinelli, Isabella},
  journal={Statistical science},
  pages={273--304},
  year={1995},
  publisher={JSTOR}
}

@inproceedings{ru2020bayesian,
  title={Bayesian optimisation over multiple continuous and categorical inputs},
  author={Ru, Binxin and Alvi, Ahsan and Nguyen, Vu and Osborne, Michael A and Roberts, Stephen},
  booktitle={International Conference on Machine Learning},
  pages={8276--8285},
  year={2020},
  organization={PMLR}
}

@inproceedings{gonzalez_Preferential_2017,
  title = {Preferential {{Bayesian}} Optimization},
  author = {Gonz{\'a}lez, Javier and Dai, Zhenwen and Damianou, Andreas and Lawrence, Neil D.},
  booktitle = {Proceedings of the 34th International Conference on Machine Learning},
  series = {Proceedings of Machine Learning Research},
  volume = {70},
  pages = {1282--1291},
  year = {2017}
}

@inproceedings{jesson_CausalBALD_2021,
  title = {Causal-{{BALD}}: Deep {{Bayesian}} Active Learning of Outcomes to Infer Treatment-Effects from Observational Data},
  author = {Jesson, Andrew and Tigas, Panagiotis and {van Amersfoort}, Joost and Kirsch, Andreas and Shalit, Uri and Gal, Yarin},
  booktitle = {Advances in Neural Information Processing Systems},
  volume = {34},
  pages = {30465--30478},
  year = {2021}
}
